\documentclass{article}

\usepackage{microtype}
\usepackage{graphicx}
\usepackage{subfigure}
\usepackage{booktabs} 

\usepackage{hyperref}

\usepackage{amsmath}
\usepackage{amssymb}
\usepackage{amsthm,url,balance}
\usepackage{color}
\usepackage[T1]{fontenc}

\usepackage{todonotes}
\usepackage{algorithm}
\usepackage{algorithmic}
\usepackage[ruled,vlined,algo2e]{algorithm2e}

\usepackage{dsfont}	
\usepackage{wrapfig}
\usepackage{footnote}

\newcommand{\ie}{{\it i.e. }}

\newcommand{\eg}{{\it e.g. }}
\newcommand{\vs}{{\it vs. }}

\newtheorem{lemma}{Lemma}

\newtheorem{thm}{Theorem}

\newtheorem{claim}{Claim}
\usepackage{bbm}
\usepackage[accepted]{icml2019}

\icmltitlerunning{Equilibrated Recurrent Neural Network}

\begin{document}

\twocolumn[
\icmltitle{Equilibrated Recurrent Neural Network: \\ Neuronal Time-Delayed Self-Feedback Improves Accuracy and Stability}



\icmlsetsymbol{equal}{*}

\begin{icmlauthorlist}
\icmlauthor{Ziming Zhang}{equal,merl}
\icmlauthor{Anil Kag}{equal,bu}
\icmlauthor{Alan Sullivan}{merl}
\icmlauthor{Venkatesh Saligrama}{bu}
\end{icmlauthorlist}

\icmlaffiliation{merl}{Mitsubishi Electric Research Laboratories (MERL), MA, USA}
\icmlaffiliation{bu}{Department of Electrical and Computer Engineering, Boston University, MA, USA}

\icmlcorrespondingauthor{Ziming Zhang}{zzhang@merl.com}

\icmlkeywords{Machine Learning, ICML}

\vskip 0.3in
]



\printAffiliationsAndNotice{\icmlEqualContribution} 

\begin{abstract}
We propose a novel {\it Equilibrated Recurrent Neural Network} (ERNN) to combat the issues of inaccuracy and instability in conventional RNNs. Drawing upon the concept of autapse in neuroscience, we propose augmenting an RNN with a time-delayed self-feedback loop. Our sole purpose is to modify the dynamics of each internal RNN state and, at any time, enforce it to evolve close to the equilibrium point associated with the input signal at that time. We show that such self-feedback helps stabilize the hidden state transitions leading to fast convergence during training while efficiently learning discriminative latent features that result in state-of-the-art results on several benchmark datasets at test-time. We propose a novel inexact Newton method to solve fixed-point conditions given model parameters for generating the latent features at each hidden state. We prove that our inexact Newton method converges locally with linear rate (under mild conditions). We leverage this result for efficient training of ERNNs based on backpropagation. 
%
%
%
%
\end{abstract}

\section{Introduction}\label{sec:intr}
Recurrent neural networks (RNNs) are useful tools to analyze sequential data, and have been widely used in many applications such as natural language processing \cite{sutskever2014sequence} and computer vision \cite{hori2017attention}. It is well-known that training RNNs is particularly challenging because we often encounter both diverging as well as vanishing gradients \cite{pascanu2013difficulty}. In this paper we propose an efficient training algorithm for RNNs that substantially improves convergence during training while achieving state-of-the-art generalization at test-time.

\begin{wrapfigure}{r}{.4\linewidth}
	\begin{center}
		\includegraphics[width=\linewidth]{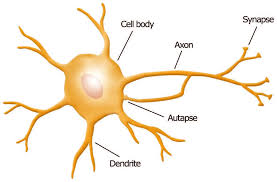}
		\vspace{-7mm}
		\caption{\footnotesize Illustration of autapse for a neuron \cite{herrmann2004autapse}.}
		\label{fig:autapse}
	\end{center}
	\vspace{-10pt}
\end{wrapfigure}
\textbf{Motivation:} We draw our inspiration from works in Neuroscience \cite{seung2000autapse}, which describe a concept called \textit{Autapse}. According to this concept, some neurons in the neocortex and hippocampus are found to enforce {\em time-delayed self-feedback} as a means to control, excite and stabilize neuronal behavior. In particular, researchers (\eg \cite{qin2015emitting}) have found that negative feedback tends to damp excitable neuronal behavior while positive feedback can excite quiescent neurons. \citet{herrmann2004autapse} have experimented with physical simulation based on artificial autapses, and more recently \citet{fan2018autapses} have demonstrated how autapses lead to enhanced synchronization resulting in better coordination and control of neuronal networks.

In this light we propose two modifications in the architecture and operation of RNNs as follows:
\begin{enumerate}\setlength\itemsep{-0.3em}
    \item[M1.] Adding self-feedback to each RNN hidden state; 
    \item[M2.] Processing RNNs with self-feedback based on upsampled inputs interpolated with a constant filter so that each input sequence can achieve a {\it set-point} over time.
\end{enumerate}
These modifications lead to our novel {\em Equilibrated Recurrent Neural Network (ERNN)}, as illustrated in Fig. \ref{fig:fp_rnn}. Compared with conventional RNNs, our ERNNs introduce a self-feedback link to each hidden state in the unfolded networks. Using the same unfolding trick, we repeat each input word as an upsampled input sequence. This step essentially generates a time-delayed hidden state towards the equilibrium point. In summary, based on the self-feedback loops, our ERNNs can be considered as {\em RNN-in-RNN} networks where the outer loop accounts for timesteps whereas the inner loop accounts for time-delay internal state dynamics.

We further demonstrate that such self-feedback helps in: 
\begin{enumerate}\setlength\itemsep{-0.3em}
    \item[H1.] Stabilizing the system that allows for equilibration of the internal state evolution;
    \item[H2.] Learning discriminative latent features;
    \item[H3.] Accelerating convergence in training;
    \item[H4.] Achieving good generalization in testing.
\end{enumerate}

\begin{figure}[t]
			\centerline{\includegraphics[clip=true,width=\linewidth]{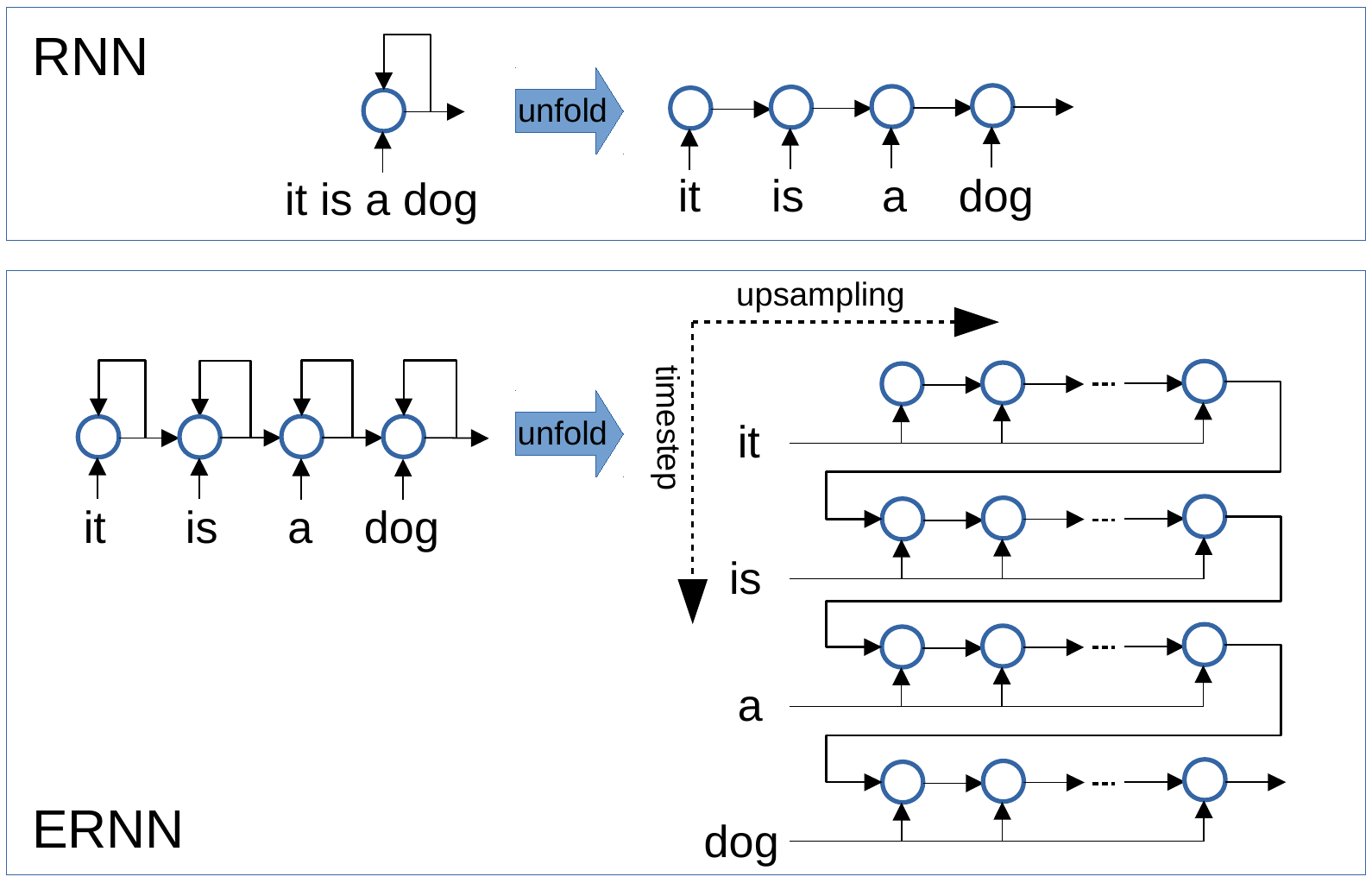}}
	\caption{\footnotesize Comparison between {\bf (top)} RNN and {\bf (bottom)} ERNN.}
	\label{fig:fp_rnn}
\end{figure}

{\bf Validation of Proposed Method on Toy Data:}
We compare conventional RNN with our ERNN in Fig. \ref{fig:fp_rnn}. We define the dynamics in the RNN as follows:
\begin{align}
    \mathbf{h}_t = \tanh\left(\mathbf{V}\mathbf{h}_{t-1} + \mathbf{W}\mathbf{x}_t + \mathbf{b}\right),
\end{align}
where $\mathbf{x}_t, \mathbf{h}_t$ denote the input data and output latent feature at $t$-th timestep, respectively, and $\mathbf{V}, \mathbf{W}, \mathbf{b}$ are the model parameters. Similarly we consider a special case of ERNN for exposition: 
\begin{align}\label{eqn:val_h}
    \mathbf{h}_t = \tanh\left(\mathbf{h}_{t} + \mathbf{V}\mathbf{h}_{t-1} + \mathbf{W}\mathbf{x}_t + \mathbf{b}\right)
\end{align}
so that both models contain the same number of parameters. Notice that here $\mathbf{h}_t$ is a fixed point of the equation \footnote{While convergence to equilibrium is outside the scope of the paper, we point to control theorists who describe conditions ~\cite{Barabanov} for these types of recurrence.}.


We generate a toy dataset to demonstrate the effectiveness of our approach. Our dataset is a collection of 2D random walks with Gaussian steps: $X_t \sim N\left(X_{t-1},\sigma_i I\right)\in\mathbb{R}^2$ with $X_0=\mathbf{0}, t \in [100]$, and $\sigma_0 = 0.1, \sigma_1 = 1$ corresponding to two classes. We generate $10^4$ walks for each class and choose half the size for training and leave the rest for testing.

Both RNN and ERNN are endowed with a 10-dimensional hidden state. As is the convention, we utilize the state in the last timestep for classification. We tune the hyper-parameters such as learning rate to achieve best model parameters for each network, where we train our ERNN using the proposed method in Sec. \ref{sec:ERNN}.

\begin{figure}[t]
	\begin{minipage}[b]{0.495\linewidth}
		\begin{center}
			\centerline{\includegraphics[clip=true,width=\linewidth]{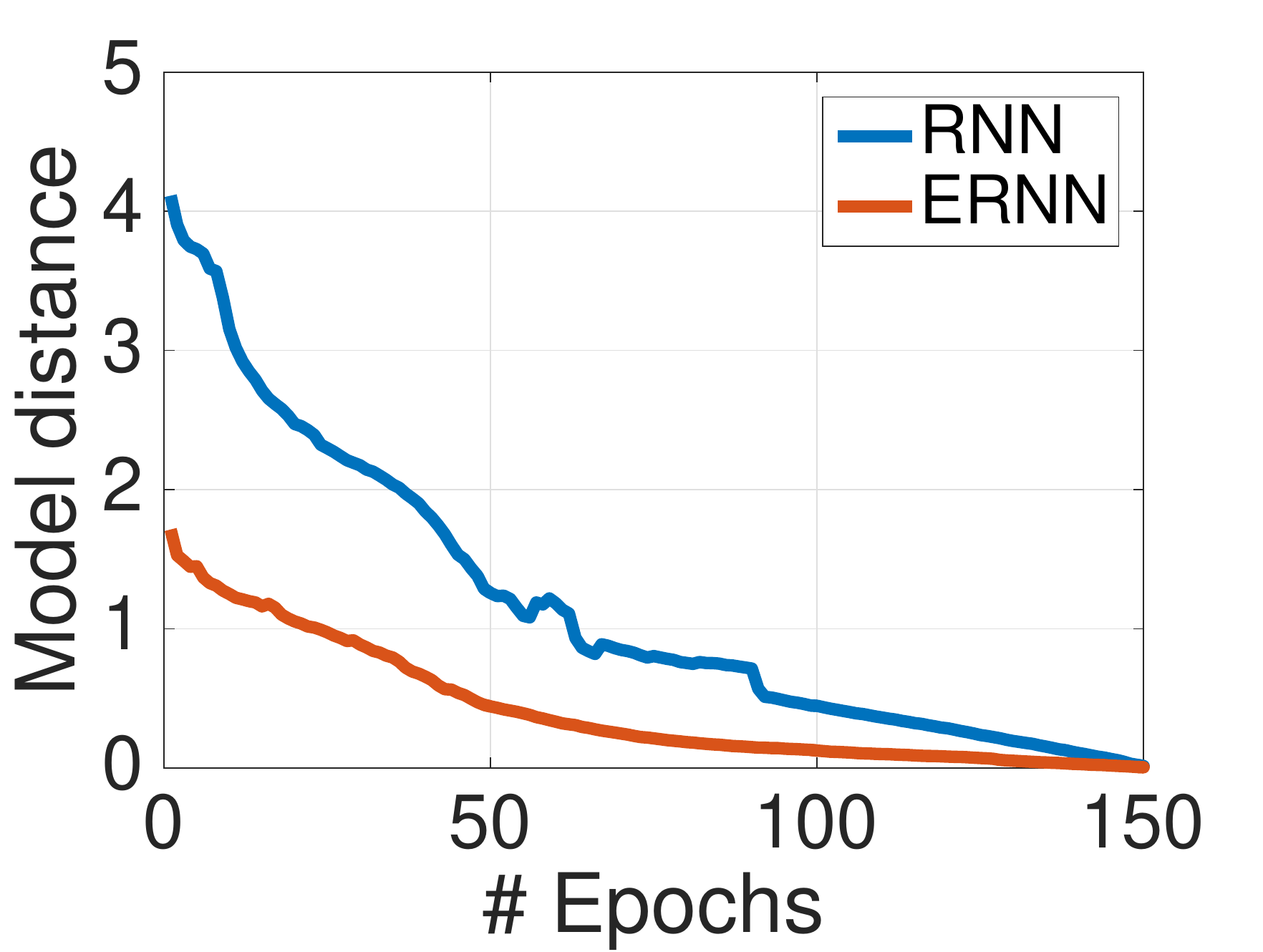}}
			\centerline{\footnotesize (a)}	
		\end{center}
	\end{minipage}
	\begin{minipage}[b]{0.495\linewidth}
		\begin{center}
			\centerline{\includegraphics[clip=true,width=\linewidth]{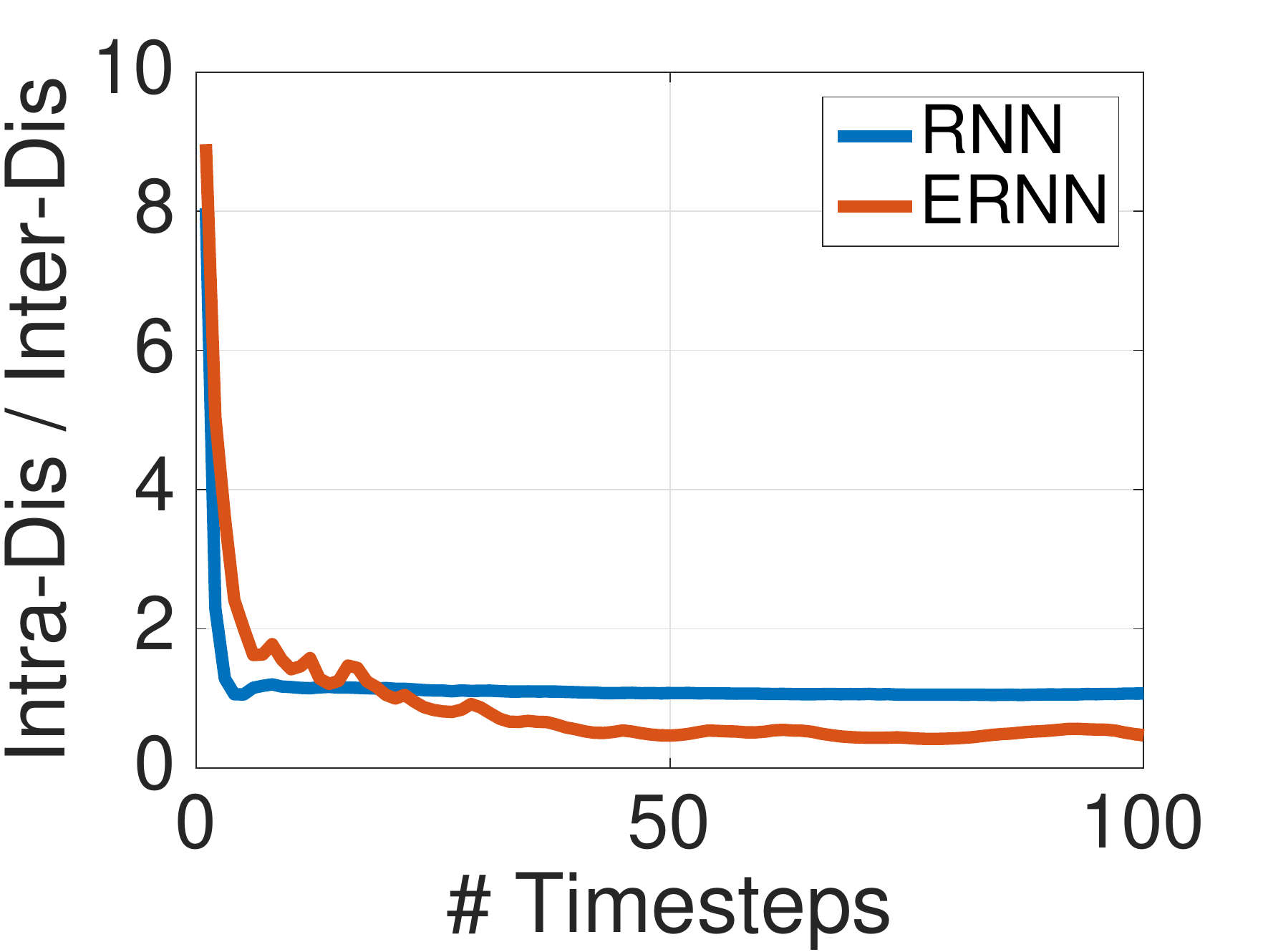}}
			\centerline{\footnotesize (b)}	
		\end{center}
	\end{minipage}
	\vspace{-10mm}
	\caption{\footnotesize Validation of {\bf (a)} H1 and {\bf (b)} H2 on toy data.}
	\label{fig:val}
\end{figure}

To validate (H1), we compare the Euclidean distances between the model after each epoch and the one after the final epoch, which we consider is convergent during training. Our results in Fig. \ref{fig:val}(a) show that the change of ERNN per epoch is always smaller than RNN, which reveals that neuronal self-feedback indeed improves the stability in hidden states.

To validate (H2), we compare the feature discrimination over timesteps after training by computing the ratio of intra-class distance and inter-class distance. Our results in Fig. \ref{fig:val}(b) show that the discrimination of ERNN gradually improves and after a period of transition is substantially superior to RNN, whose discrimination saturates very quickly. 

\begin{wrapfigure}{r}{.43\linewidth}
	\vspace{-18pt}
	\begin{center}
		\includegraphics[width=\linewidth]{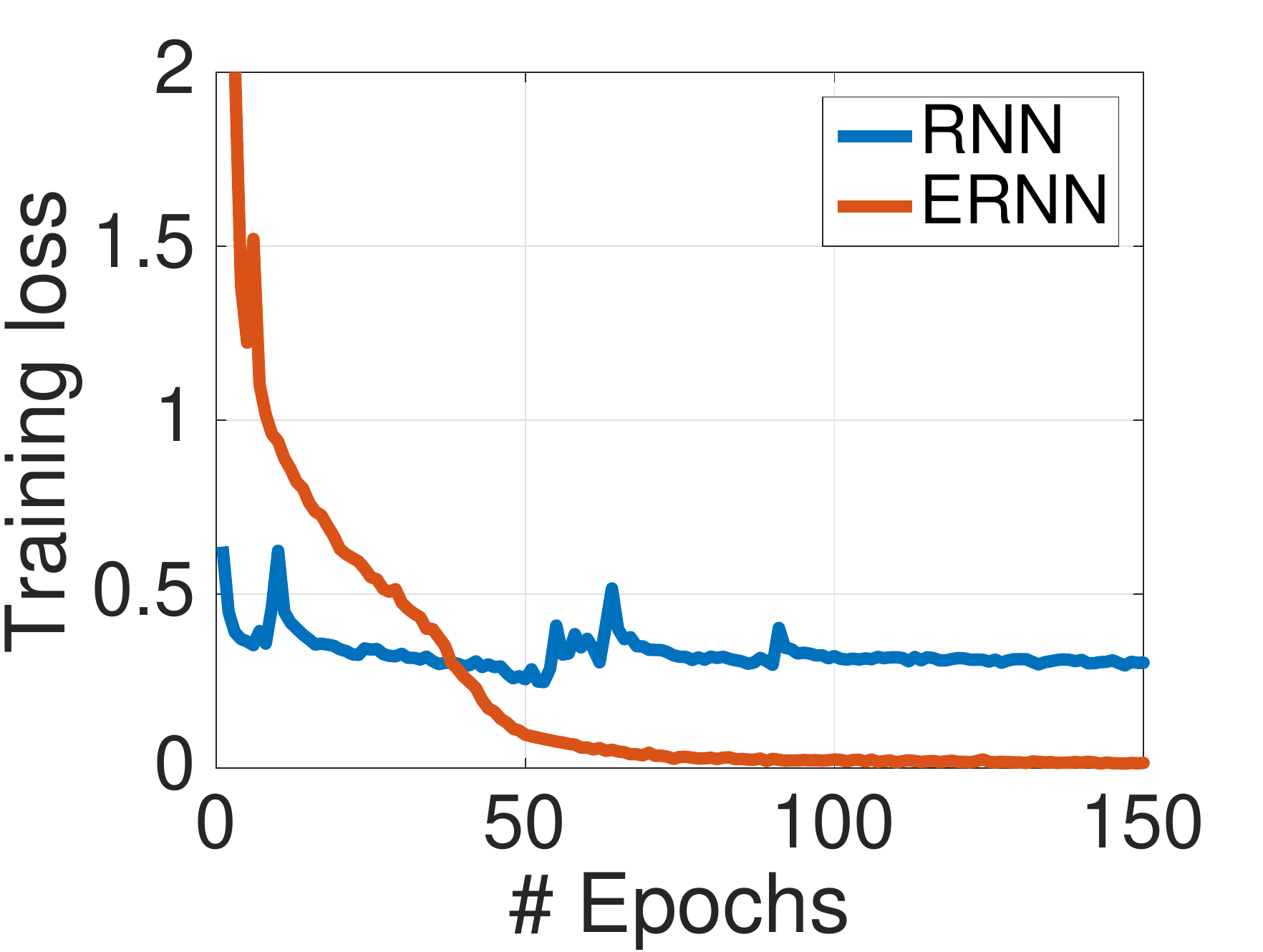}
		\vspace{-7mm}
		\caption{\footnotesize Validation of H3.}
		\label{fig:H3}
	\end{center}
	\vspace{-15pt}
\end{wrapfigure}
To validate (H3), we show the training losses of each network over epochs in Fig.~\ref{fig:H3}. As we see, ERNN converges to lower values. At test-time, RNN and ERNN achieve accuracy of 86.6\% and 99.7\%, respectively, validating (H4).

\textbf{Vanishing and Exploding Gradients:} In light of these results, we speculate how autapse could possibly combat this issue. While we have an incomplete understanding of ERNNs, we will provide some intuition. Fixed points can be iteratively approximated with our inexact Newton method. This method involves setting the next iterate to be equal to previous iterate and a weighted error term. We refer to Sec~\ref{ssec:network} for further details. In essence, this iteration implicitly encodes skip/residual connections, which have been shown recently to overcome vanishing and exploding gradient issue~\cite{kusupati2018nips}. 

\begin{wrapfigure}{r}{.5\linewidth}
	\vspace{-15pt}
	\begin{center}
		\includegraphics[width=\linewidth]{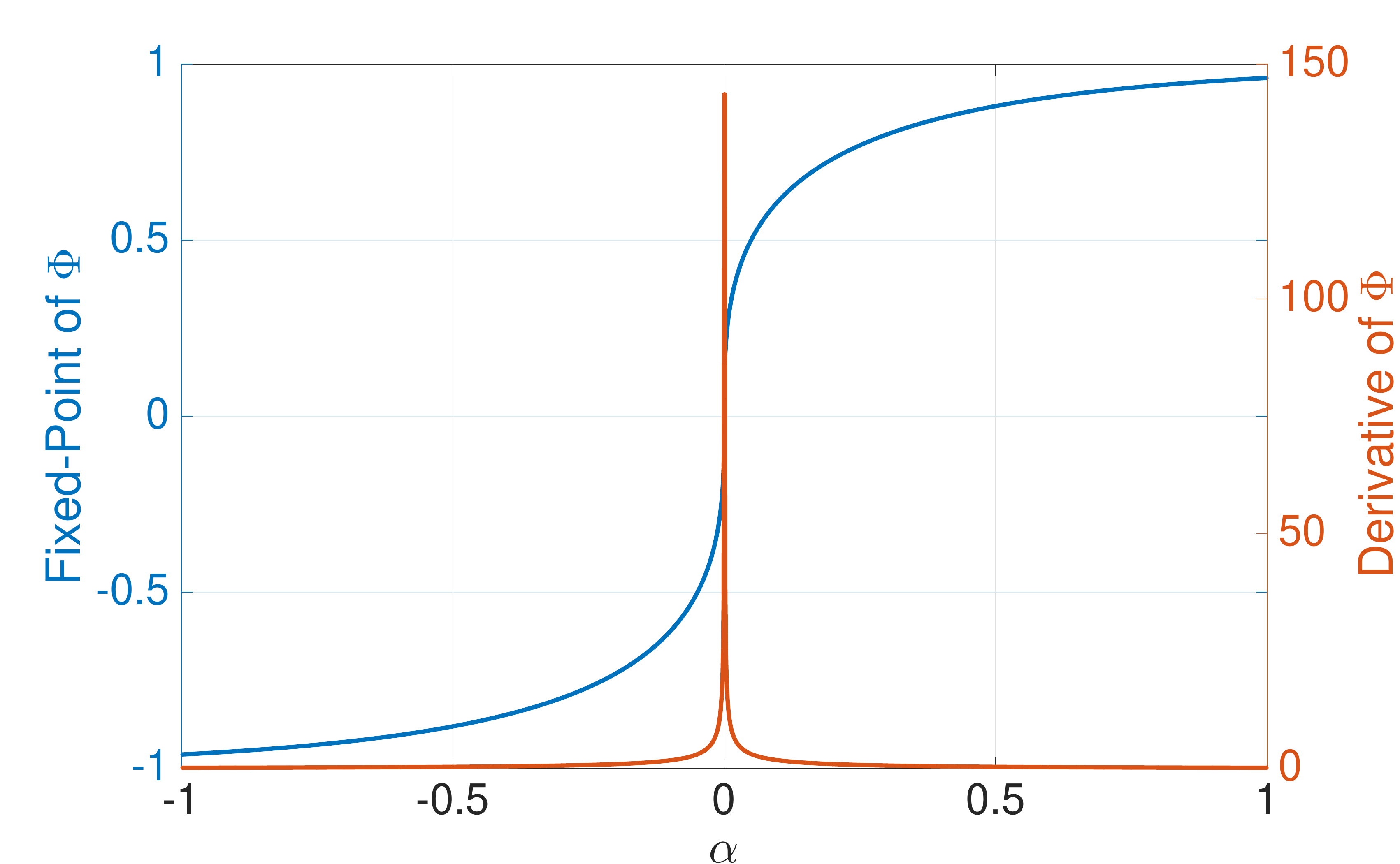}
		\vspace{-7mm}
		\caption{\footnotesize Illustration of fixed-point and derivative of $\Phi$ \vs $\alpha$}
		\label{fig:derivativefig}
	\end{center}
	\vspace{-10pt}
\end{wrapfigure}
A different argument is based on defining the equilibrium point for Eq.~\ref{eqn:val_h} in terms of an abstract map, $\Phi$, that satisfies $\mathbf{h}_t=\Phi(\mathbf{V}\mathbf{h}_{t-1} + \mathbf{W}\mathbf{x}_t + \mathbf{b})$. In this case, the fixed points are entry-wisely independent and thus we can consider each entry in the fixed points in isolation. As numerically illustrated in Fig. \ref{fig:derivativefig}, we find that the equilibrium point of $h=\tanh(h+\alpha)$ is some function $h=\Phi(\alpha)$, where $\alpha=[\mathbf{V}\mathbf{h}_{t-1} + \mathbf{W}\mathbf{x}_t + \mathbf{b}]_j$ is the $j$-th entry. Theoretically the derivative of $\Phi$ at $\alpha=0$ is infinity, \ie unbounded. This observation violates the bounded conditions in \cite{pascanu2013difficulty} for the existence of vanishing gradients. We conjecture that this point lies at the root of why vanishing gradients in training does not arise. 

While we do not completely understand the issue of exploding gradients, we conjecture that our method learns the parameters, \ie $\mathbf{V}, \mathbf{W}, \mathbf{b}$ for $\Phi$, such that the $\alpha$ equal to zero is not a local optima and in this way ERNN is able to handle exploding gradients as well. In fact, in our experiments we did not encounter exploding gradients as an issue.


{\bf Contributions:}
Our key contributions are as follows:
\begin{enumerate}\setlength\itemsep{-0.3em}
    \item[C1.] We propose a novel Equilibrated Recurrent Neural Network (ERNN) as defined in Sec. \ref{ssec:problem_definition} drawing upon the concept of autapse from the Neuroscience literature. Our method improves training stability of RNNs as well as the accuracy.
    \item[C2.] We propose a novel inexact Newton method for fixed-point iteration given learned models in Sec. \ref{ssec:newton}. We prove that under mild conditions our fixed-point iteration algorithm can converge locally at linear rate.
    \item[C3.] We propose a principled way in Sec.~\ref{ssec:network} to convert ERNNs into trainable networks with residual connections using backpropagation. We demonstrate that FastRNN \cite{NIPS2018_8116} is a special ERNN.
    \item[C4.] We present an exemplar ERNN in Sec. \ref{ssec:example}. In Sec.~\ref{sec:exp} we compare this model with other state-of-the-art RNNs empirically to show its superiority in terms of accuracy as well as training efficiency.
\end{enumerate}

\section{Related Work}
We summarize related work from the following aspects:


{\bf Optimizers:}
Backpropagation through time (BPTT) \cite{werbos1990backpropagation} is a generalization of backpropagation that can compute gradients in RNNs but suffer from large storage of hidden states in long sequences. Its truncated counterpart, truncated backpropagation through time (TBPTT) \cite{jaeger2002tutorial}, is widely used in practice but suffers from learning long-term dependencies due to the truncation bias. The Real Time Recurrent Learning algorithm (RTRL) \cite{williams1989learning} addresses this issue at the cost of high computational requirement. Recently several papers address the computational issue in the RNN optimizers, for instance, from the perspective of effectiveness \cite{martens2011learning, cho2015hessian}, storage \cite{NIPS2016_6221, tallec2017unbiased, NIPS2018_7894, mackay2018reversible, liao2018reviving} or parallelization \cite{bradbury2016quasi, lei2017training, martin2018parallelizing}.

\textbf{Feedforward vs. Recurrent:}
The popularity of recurrent models stems from the fact that they are particularly well-suited for sequential data, which exhibits long-term temporal dependencies. 
Nevertheless, an emerging line of research has highlighted several shortcomings of RNNs. In particular, apart from the well-known issues of exploding and vanishing gradients, several authors~\cite{Oord2016WaveNetAG,Gehring2017ConvolutionalST,Vaswani2017AttentionIA,Dauphin2017LanguageMW} have observed that sequential processing leads to large training and inference costs because RNNs are inherently not parallelizable. To overcome these shortcomings they 
have proposed methods that replace recurrent models with parallelizable feedforward models which truncate the receptive field, and such feedforward structures have begun to show promise in a number of applications. Motivated by these works, \citet{miller2018recurrent} have attempted to theoretically justify these findings. In particular, their work shows that under strong assumptions, namely, the class of so-called ``stable'' RNN models, recurrent models can be well-approximated by feedforward structures with a relatively small receptive field.
Nevertheless, the assumption that RNNs are stable appears to be too strong as we have seen in a number of our experiments and so it does not appear possible to justify the usage of limited receptive field and feedforward networks in a straightforward manner.


In contrast, our paper does not attempt to a priori limit the receptive field. Indeed, the equilibrium states derived here are necessarily functions of both the input and the prior state and they influence the equilibrium solution since we are dealing with an underlying non-linear dynamical system. Instead, we show that ERNNs operating near equilibrium lend themselves to an inexact Newton method, whose sole purpose is to force the state towards equilibrium solution. 
In summary, although, there are parallels between our work and the related feedforward literature, this similarity appears to be coincidental and superficial. 

{\bf Architectures:}
Our work can also be related to a number of related works that attempt to modify RNNs to improve the issues arising from exploding and vanishing gradients.

Long short-term memory (LSTM) \cite{hochreiter1997long} is widely used in RNNs to model long-term dependency in sequential data. Gated recurrent unit (GRU) \cite{cho2014properties} is another gating mechanism that has been demonstrated to achieve similar performance of LSTM with fewer parameters. Unitary RNNs \cite{arjovsky2016unitary, jing2017tunable} is another family of RNNs that consist of well-conditioned state transition matrices.

Recently residual connections have been applied to RNNs with remarkable improvement of accuracy and efficiency for learning long-term dependency. For instance, \citet{chang2017dilated} proposed dilated recurrent skip connections. \citet{campos2017skip} proposed Skip RNN to learn to skip state updates. \citet{kusupati2018nips} proposed FastRNN by adding a residual connection to handle inaccurate training and inefficient prediction. They further proposed FastGRNN by extending the residual connection to a gate, which involves low-rank approximation, sparsity, and quantization (LSQ) as well to reduce model size.

In contrast to these works, 
our work is based on enforcing RNNs through time-delayed self-feedback. Our work leverages an inexact Newton method for efficient training of ERNNs. Our method generalizes FastRNN in \cite{kusupati2018nips} in this context.

\section{Equilibrated Recurrent Neural Network}\label{sec:ERNN}
\subsection{Problem Definition}\label{ssec:problem_definition}
Without loss of generality, we consider our ERNNs in the context of supervised learning. That is,
\begin{align}\label{eqn:h}
    \min_{\phi\in\Phi, \omega\in\Omega} & \sum_i\ell(\mathbf{h}_{i,T}, y_i; \omega),  \\
    \mbox{s.t.}\quad & \mathbf{h}_{i,t} = \phi(\mathbf{h}_{i,t}, \mathbf{h}_{i,t-1}, \mathbf{x}_{i,t}), \forall i, \forall t\in[T]. \nonumber
\end{align}
Here $\{(x_i, y_i)\}$ denotes a collection of training data where $x_i=\{\mathbf{x}_{i,t}\}_{t=1}^T\subseteq\mathbb{R}^d, \forall i$ denotes the $i$-th training sample consisting of $T$ timesteps and $y_i$ denotes its associated label. $\phi:\mathbb{R}^n\times\mathbb{R}^n\times\mathbb{R}^d\rightarrow\mathbb{R}^n$ denotes a (probably nonconvex) differentiable state-transition mapping function learned from a feasible solution space $\Phi$, $\mathbf{h}_{i,t}\in\mathbb{R}^n, \forall i, \forall t$ denotes a hidden state for sample $i$ at time $t$, and $\ell$ denotes a loss function parameterized by $\omega$ that is also learned from a feasible solution space $\Omega$. For notational simplicity we assume hidden states all have the same dimension, though our approach can easily be generalized. 

\subsubsection{Two-Dimensional Space-Time RNNs}
To solve this problem, we view the proposed method as two recurrences, one in space and the other in time. We introduce a space variable $k$ and consider the following recursion at any fixed time $t$:
\begin{align}\label{eqn:space}
\mathbf{h}_{i,t}^{(k)} = \phi\left(\mathbf{h}_{i,t}^{(k-1)}, \mathbf{h}_{i,t-1},  \mathbf{x}_{i,t}\right),\, k=[K], \mathbf{h}_{i,t}^{(0)}=\mathbf{h}_{i,t-1}.
\end{align}
We then have a second RNN in time by moving to the next timestep with $\mathbf{h}_{i,t}=\mathbf{h}_{i,t}^{(K)}$ and then repeating Eq.~\ref{eqn:space}.

Nevertheless, in order to implement this approach we need to account for two issues. First, the space recursion as written may not converge. To deal with this we consider Newton's method and suitably modify our recursion in the sequel. Second, we need to transform the updates into a form that lends itself to updates using backpropagation. We discuss as well in the sequel. 


\begin{figure*}[t]
    \centering
    \includegraphics[width=\linewidth]{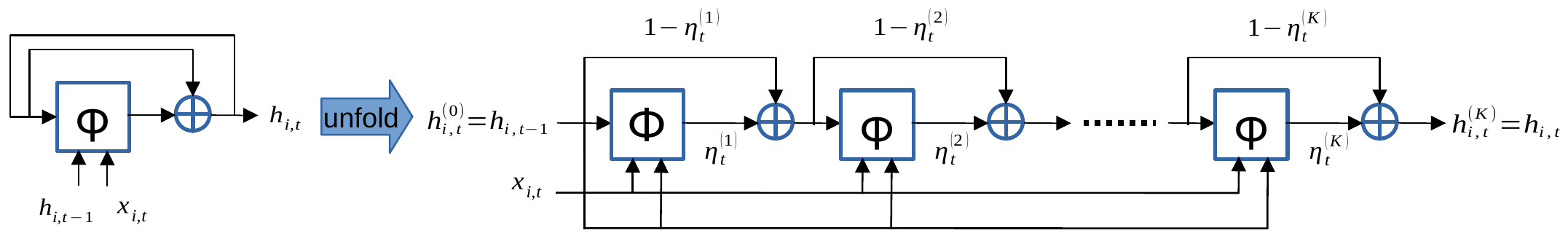}
    \vspace{-2mm}
   \caption{\footnotesize Illustration of networks for computing $\mathbf{h}_{i,t}$ using Eq. \ref{eqn:resRNN}.}
   \label{fig:resRNN}
\end{figure*}

\subsection{Inexact Newton Method for Fixed-Point Iteration}\label{ssec:newton}
\subsubsection{Algorithm}
{\em Newton's method} is a classic algorithm for solving the system of (nonlinear) equations, say $f(\mathbf{z})=\mathbf{0}$, as follows:
\begin{align}
    \mathbf{z}^{(k+1)} = \mathbf{z}^{(k)} + \mathbf{s}^{(k)}, \, f'\left(\mathbf{z}^{(k)}\right)\mathbf{s}^{(k)} = -f\left(\mathbf{z}^{(k)}\right), 
\end{align}
where $f'$ denotes the gradient of function $f$. The large number of unknowns in the equation system and the nonexistence of feasible solutions, however, may lead to very expensive updates in Newton's method.

{\em Inexact Newton methods}, instead, refer to a family of algorithms that aims to solve the equation system $f(\mathbf{z})=\mathbf{0}$ approximately at each iteration using the following rule:
\begin{align}
    \mathbf{z}^{(k+1)} = \mathbf{z}^{(k)} + \mathbf{s}^{(k)}, \, f'\left(\mathbf{z}^{(k)}\right)\mathbf{s}^{(k)} = -f\left(\mathbf{z}^{(k)}\right)+\mathbf{r}^{(k)},     
\end{align}
where $\mathbf{r}^{(k)}$ denotes the error at the $k$-th iteration between $f\left(\mathbf{z}^{(k)}\right)$ and $\mathbf{0}$. Such algorithms provide the potential of updating the solutions efficiently with (local) convergence guarantee under certain conditions \cite{dembo1982inexact}.

Inspired by inexact Newton methods and the Barzilai-Borwein method \cite{barzilai1988two}, we propose a new inexact Newton method to solve $f(\mathbf{z})=\mathbf{0}$ as follows: 
\begin{align}\label{eqn:z}
    \mathbf{z}^{(k+1)} = \mathbf{z}^{(k)} + \eta^{(k)}f\left(\mathbf{z}^{(k)}\right), \, \eta_t\in\mathbb{R},
\end{align}
where we intentionally set $\mathbf{s}^{(k)}=\eta^{(k)}f\left(\mathbf{z}^{(k)}\right)$, $\mathbf{r}^{(k)}=\left[\mathbf{I}+\eta^{(k)}f'\left(\mathbf{z}^{(k)}\right)\right]f\left(\mathbf{z}^{(k)}\right)$ and $\mathbf{I}$ is an identity matrix.

\subsubsection{Convergence Analysis}
We analyze the convergence of our proposed inexact Newton method in Eq. \ref{eqn:z}. We prove that under certain condition our method can converge locally with linear convergence rate.

\begin{lemma}[\cite{dembo1982inexact}]\label{lem:1}
Assume that $\frac{\|\mathbf{r}^{(k)}\|}{\|f(\mathbf{z}^{(k)})\|}\leq\tau<1, \forall k$ where $\|\cdot\|$ denotes an arbitrary norm and the induced operator norm. There exists $\varepsilon>0$ such that, if $\left\|\mathbf{z}^{(0)}-\mathbf{z}^*\right\|\leq\varepsilon$, then the sequence of inexact Newton iterates $\left\{\mathbf{z}^{(k)}\right\}$ converges to $\mathbf{z}^*$. Moreover, the convergence is linear in the sense that $\left\|\mathbf{z}^{(k+1)}-\mathbf{z}^*\right\|_*\leq \tau\left\|\mathbf{z}^{(k)}-\mathbf{z}^*\right\|_*$,
where $\|\mathbf{y}\|_*=\|f'(\mathbf{z}^*)\mathbf{y}\|$.
\end{lemma}

\begin{thm}\label{thm:convergence}
Assume that $\left\|\mathbf{I}+\eta^{(k)}f'\left(\mathbf{z}^{(k)}\right)\right\|<1, \forall k$. There exists $\varepsilon>0$ such that, if $\left\|\mathbf{z}^{(0)}-\mathbf{z}^*\right\|\leq\varepsilon$, then the sequence $\left\{\mathbf{z}^{(k)}\right\}$ generated using Eq. \ref{eqn:z} converges to $\mathbf{z}^*$ locally with linear convergence rate.
\end{thm}
\begin{proof}
Based on Eq. \ref{eqn:z} and the assumption, we have
\begin{align}
    \hspace{-2mm}\frac{\left\|\mathbf{r}^{(k)}\right\|}{\left\|f\left(\mathbf{z}^{(k)}\right)\right\|}&\leq\frac{\left\|\mathbf{I}+\eta^{(k)}f'\left(\mathbf{z}^{(k)}\right)\right\|\left\|f\left(\mathbf{z}^{(k)}\right)\right\|}{\left\|f\left(\mathbf{z}^{(k)}\right)\right\|} < 1.
\end{align}
Further based on Lemma \ref{lem:1}, we complete our proof.
\end{proof}


{\bf Discussion:}
The condition of $\left\|\mathbf{I}+\eta^{(k)}f'\left(\mathbf{z}^{(k)}\right)\right\|<1, \forall k$ in Thm. \ref{thm:convergence} suggests that $\eta^{(k)}$ can be determined dependently (and probably differently) over the iterations. Also the linear convergence rate in Thm.~\ref{thm:convergence} indicates that the top iterations in Eq. \ref{eqn:z} are more important for convergence, as the difference between the current solution and the optimum decreases exponentially. This is the main reason that we can control the number of iterations in our algorithm (even just once) so that the convergence behavior can still be preserved.

Notice that if $\|\cdot\|$ in Thm. \ref{thm:convergence} denotes $\ell_2$ norm, the condition of $\left\|\mathbf{I}+\eta^{(k)}f'\left(\mathbf{z}^{(k)}\right)\right\|<1, \forall k$ essentially defines a lower and upper bounds for the eigenvalue of the matrix $\mathbf{I}+\eta^{(k)}f'\left(\mathbf{z}^{(k)}\right)$. As we show in our experiments later, $\eta^{(k)}$ usually is very small, indicating that the range of the spectrum of matrix $f'\left(\mathbf{z}^{(k)}\right)$ is allowed to be quite large. This observation significantly increases the probability of the condition being feasible in practice.

\subsection{Approximating Fixed-Point Iteration with Residual Connections}\label{ssec:network}
Let us consider solving $f(\mathbf{h}_{i,t}) \equiv \phi\left(\mathbf{h}_{i,t}, \mathbf{h}_{i,t-1}, \mathbf{x}_{i,t}\right) - \mathbf{h}_{i,t} = \mathbf{0}$ based on Eq. \ref{eqn:z}. By substitute $f(\mathbf{h}_{i,t})$ into Eq. \ref{eqn:z}, we have the following update rule:
\begin{align}\label{eqn:resRNN}
    & \mathbf{h}_{i,t}^{(k)} = \mathbf{h}_{i,t}^{(k-1)} + \eta_t^{(k)}\Big[\phi\left(\mathbf{h}_{i,t}^{(k-1)}, \mathbf{h}_{i,t-1},  \mathbf{x}_{i,t}\right) - \mathbf{h}_{i,t}^{(k-1)}\Big] \nonumber \\
    & = \left(1-\eta_t^{(k)}\right)\mathbf{h}_{i,t}^{(k-1)} + \eta_t^{(k)}\phi\left(\mathbf{h}_{i,t}^{(k-1)}, \mathbf{h}_{i,t-1}, \mathbf{x}_{i,t}\right).
\end{align}

This update rule can be efficiently implemented using residual connections in networks, as illustrated in Fig. \ref{fig:resRNN}, where $\oplus$ denotes the entry-wise plus operator, numbers associated with arrow lines denote the weights for linear combination, each blue box denotes a same sub-network for representing function $\phi$ that accounts for the residual. 

During training, we learn the parameters in $\phi$ as well as all the $\eta$'s for linear combination so that the learned features are good for supervision. To do so, we predefine the number of $K$ for each time $t$, and concatenate such networks together with the supervision signals. Then we can apply backpropagation to minimize the total loss, same as conventional feedforward neural networks. 


\begin{figure*}[t]
	\begin{minipage}[b]{0.45\linewidth}
		\begin{center}
			\centerline{\includegraphics[clip=true,width=.9\linewidth]{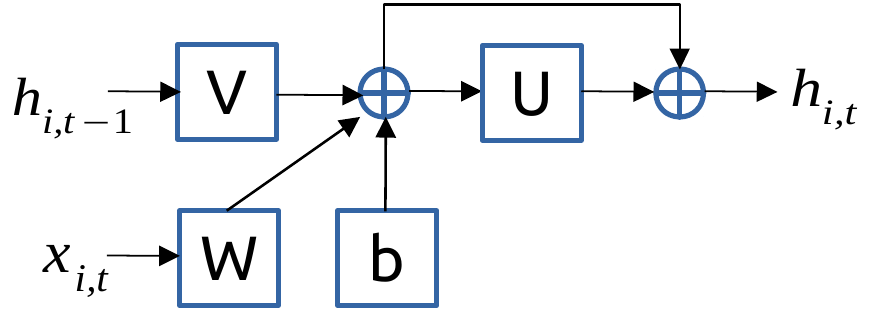}}	
			\centerline{\footnotesize (a) Linear}	
		\end{center}
	\end{minipage}
	\begin{minipage}[b]{0.55\linewidth}
		\begin{center}
			\centerline{\includegraphics[clip=true,width=.9\linewidth]{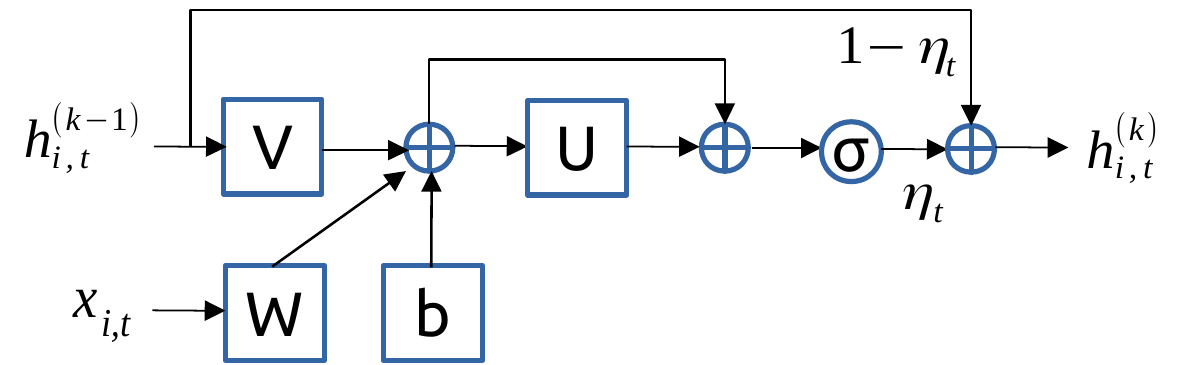}}	
			\centerline{\footnotesize (b) Embedded nonlinear}	
		\end{center}
	\end{minipage}
	\vspace{-7mm}
	\caption{\footnotesize Illustration of networks for computing linear RNNs and the embedded nonlinear.}
	\label{fig:linear_RNN}
\end{figure*}

\subsection{Example: Linear Dynamical Systems and Beyond}\label{ssec:example}

Now let us consider the linear dynamic systems for modelling $\phi$ in Eq. \ref{eqn:h}, defined as follows:
\begin{align}\label{eqn:linear}
    \phi(\mathbf{h}_{i,t}, \mathbf{h}_{i,t-1}, \mathbf{x}_{i,t}) = \mathbf{U}\mathbf{h}_{i,t} + \mathbf{V}\mathbf{h}_{i,t-1} + \mathbf{W}\mathbf{x}_{i,t} + \mathbf{b}, 
\end{align}
where $\mathbf{U}, \mathbf{V}\in\mathbb{R}^{n\times n}, \mathbf{W}\in\mathbb{R}^{n\times d},  \mathbf{b}\in\mathbb{R}^n$ are the parameters that need to be learned. By plugging Eq. \ref{eqn:linear} into Fig. \ref{fig:resRNN}, we can compute $\mathbf{h}_{i,t}$, but at the cost of high computation.

Suppose that the matrix $\left(\mathbf{I} - \mathbf{U}\right)\in\mathbb{R}^{n\times n}$ is invertible. We then have a close-form fixed-point solution for Eq. \ref{eqn:linear}, \ie $\mathbf{h}_{i,t}=\left(\mathbf{I} - \mathbf{U}\right)^{-1}\left(\mathbf{V}\mathbf{h}_{i,t-1} + \mathbf{W}\mathbf{x}_{i,t} + \mathbf{b}\right)$. Computing $\left(\mathbf{I} - \mathbf{U}\right)^{-1}$ using neural networks is challenging. Instead we approximate it as $\left(\mathbf{I} - \mathbf{U}\right)^{-1} = \sum_{k=0}^{\infty}\mathbf{U}^k \approx \mathbf{I} + \mathbf{U}$. Accordingly, we have an analytic solution for $\mathbf{h}_{i,t}$ as
\begin{align}\label{eqn:linear_approx}
    \mathbf{h}_{i,t}=\left(\mathbf{I} + \mathbf{U}\right)\left(\mathbf{V}\mathbf{h}_{i,t-1} + \mathbf{W}\mathbf{x}_{i,t} + \mathbf{b}\right),
\end{align}
which can be computed efficiently based on the network in Fig.~\ref{fig:linear_RNN}(a). 
This linear recursion leads to a final hidden state for classification, which is a linear time invariant convolution of input timesteps.
The discrimination of linear models for classification is very limited. To improve it, we propose a complex embedded nonlinear function for modelling $\phi$ based on linear dynamical systems that can be easily realized using the networks in Fig. \ref{fig:linear_RNN}(b), where $\sigma$ denotes a nonlinear activation function such as $\tanh$ or ReLU. Mathematically this embedding networks for supervised learning aims to minimize (approximately) the following objective:
\begin{align}\label{eqn:jordan}
    \min_{\mathbf{U}, \mathbf{V}, \mathbf{W}, \mathbf{b}, \omega\in\Omega} & \sum_i\ell(\mathbf{h}_{i,T}, y_i; \omega), \\
    \mbox{s.t.}\quad\quad & \hspace{-5mm} \mathbf{g}_{i,t}^{(k)} = \mathbf{U}\mathbf{g}_{i,t}^{(k)} + \mathbf{V}\mathbf{h}_{i,t}^{(k-1)} + \mathbf{W}\mathbf{x}_{i,t} + \mathbf{b}, \nonumber \\
    & \hspace{-5mm} \mathbf{h}_{i,t}^{(k)} = \sigma\left(\mathbf{g}_{i,t}^{(k)}\right), \nonumber \\
    & \hspace{-5mm} \mathbf{h}_{i,t}^{(0)} = \mathbf{h}_{i,t-1}, \mathbf{h}_{i,t} = \mathbf{h}_{i,t}^{(K)}, \forall i, \forall t\in[T]. \nonumber 
\end{align}
The conditions above are essentially a special case of classic Jordan networks \cite{jordan1997serial}.

{\bf Discussion:}
From the perspective of autapse, here the matrix $\mathbf{U}$ can be considered to mimics the functionality of excitation and inhibition. From the perspective of learnable features, the matrix $(\mathbf{I}+\mathbf{U})$ is to perform the {\em alignment} in the hidden state space to reduce the variance among the features. Similar ideas have been explored in other domains. For instance, PointNet \cite{qi2017pointnet} is a powerful network for 3D point cloud classification and segmentation, where there exist so-called T-Nets (transformation networks) that transform input features to be robust to noise. 

It is worth of mentioning that from the perspective of formulation, FastRNN \cite{NIPS2018_8116} can be considered as an approximate solver with $K=1$ for the same minimization problem in Eq. \ref{eqn:jordan} with a fixed $\mathbf{U}=\mathbf{0}$. Therefore, all the proofs in \cite{NIPS2018_8116} for FastRNN hold as well for ours in this special case.
\begin{claim}
In passing, while we do not formally show this here, we point out that solving Eq. \ref{eqn:jordan} using our inexact Newton method with Eq. \ref{eqn:linear_approx} and $K=1$, the bounds of our generalization error and convergence can be verified to be no larger than that of FastRNN using the same techniques as in that paper~\cite{kusupati2018nips}.
\end{claim}
This claim is empirically validated in our experiments.

\section{Experiments}\label{sec:exp}
To demonstrate our approach, we refer to ERNN in our experiments as the special one solving Eq.~\ref{eqn:jordan} using Fig.~\ref{fig:linear_RNN}(b). By default we set $K=1$ without explicit mention.

{\bf Datasets:} ERNN's performance was benchmarked on the mix of IoT and traditional RNN tasks. IoT tasks include: (a) Google-30 \cite{warden2018google30} and Google-12, \ie detection of utterances of 30 and 10 commands plus background noise and silence and (b) HAR-2 \cite{Anguita2012HAR} and DSA-19 \cite{Altun2010DSA19}, \ie Human Activity Recognition (HAR) from an accelerometer and gyroscope on a Samsung Galaxy S3 smartphone and Daily and Sports Activity (DSA) detection from a resource-constrained IoT wearable device with 5 Xsens MTx sensors having accelerometers, gyroscopes and magnetometers on the torso and four limbs. Traditional RNN tasks includes tasks such as language modeling on the Penn Treebank (PTB) dataset \cite{McAuley2013PTB}, star rating prediction on a scale of 1 to 5 of Yelp reviews  \cite{Yelp2017} and classification of MNIST images on a pixel-by-pixel sequence \cite{Lecun98gradient-basedlearning}.

All the datasets are publicly available and their pre-processing and feature extraction details are provided in \cite{NIPS2018_8116}. The publicly provided training set for each dataset was subdivided into 80\% for training and 20\% for validation. Once the hyperparameters had been fixed, the algorithms were trained on the full training set and the results were reported on the publicly available test set. Table \ref{table:1} lists the statistics of all the datasets.

\begin{table}[t]
\centering
\caption{Dataset Statistics}
\vspace{2mm}
\setlength{\tabcolsep}{3pt}
\begin{tabular}{|ccccc|} 
 \hline
 Dataset & \#Train & \#Fts & \#Steps & \#Test \\ [0.5ex] 
 \hline\hline
 Google-12 & 22,246 & 3,168 & 99 & 3,081 \\ 
 Google-30 & 51,088 & 3,168 & 99 & 6,835 \\
 Yelp-5 & 500,000 & 38,400 & 300 & 500,000 \\
 HAR-2 & 7,352 & 1,152 & 128 & 2,947 \\
 Pixel-MNIST-10 & 60,000 & 784 & 784 & 10,000 \\
 PTB-10000 & 929,589 & - & 300 & 82,430 \\
 DSA-19 & 4,560 & 5,625 & 125 & 4,560 \\
 \hline
\end{tabular}
\vspace{-3mm}
\label{table:1}
\end{table}

{\bf Baseline Algorithms and Implementation:} We compared ERNN with standard RNN, SpectralRNN \cite{2018SpectralRNN}, EURNN \cite{2017EURNN}, LSTM \cite{hochreiter1997long}, GRU \cite{cho2014properties}, UGRNN \cite{2016arXivUGRNN}, FastRNN and FastGRNN-LSQ (\ie FastGRNN without model compression but achieving better accuracy and lower training time) \cite{NIPS2018_8116}. 
Since reducing model size is not our concern, we did not pursue model compression experiments and thus did not compare ERNN with FastGRNN directly, though potentially all the compression techniques in FastGRNN could be applicable to ERNN as well.

We used the publicly available implementation \cite{EdgeML2018} for FastRNN and FastGRNN-LSQ. Except for FastRNN and FastGRNN-LSQ that we reproduced the results mentioned by verifying the hyper-parameter settings, we simply cited the corresponding numbers for the other competitors from \cite{NIPS2018_8116}. All the experiments were run on a Nvidia GTX 1080 GPU with CUDA 9 and cuDNN 7.0 on a machine with Intel Xeon 2.60 GHz GPU with 20 cores. We found that FastRNN and FastGRNN-LSQ can be trained to perform similar accuracy as reported in \cite{NIPS2018_8116} using slightly longer training time on our machine. This indicates that potentially all the other competitors can achieve similar accuracy using longer training time as well.


{\bf Hyper-parameters:} The hyper-parameters of each algorithm were set by a fine-grained validation wherever possible or according to the settings published in \cite{NIPS2018_8116} otherwise. Both the learning rate and $\eta$'s were initialized to $10^{-2}$. Since ERNN converged much faster in comparison to FastRNN and FastGRNN-LSQ, the learning rate was halved periodically, where the period was learnt based on the validation set. Replicating this on FastRNN or FastGRNN-LSQ does not achieve the maximum accuracy reported in the paper. The batch size of $128$ seems to work well across all the datasets. ERNN used ReLU as the non-linearity and Adam \cite{2015Adam} as the optimizer for all the experiments. 

{\bf Evaluation Criteria:} The primary focus in this paper is on achieving better results than state-of-the-art RNNs with much better convergence rate. For this purpose we reported model size, training time and accuracy (perplexity on the PTB dataset). Following the lines of FastRNN and FastGRNN-LSQ, for PTB and Yelp datasets, model size excludes the word-vector embedding storage.

\begin{table}[t]\footnotesize
\centering
\caption{PTB Layer Modeling - 1 Layer}
\vspace{2mm}
\setlength{\tabcolsep}{3pt}
\begin{tabular}{|c c c c|} 
 \hline
 Algorithm & \begin{tabular}[c]{@{}c@{}}Test Perplexity\end{tabular} & \begin{tabular}[c]{@{}c@{}}Model\\ Size (KB)\end{tabular} & \begin{tabular}[c]{@{}c@{}}Train\\ Time (min)\end{tabular} \\ [0.5ex] 
 \hline\hline
 FastRNN & 127.76 & 513 & 11.20 \\  
 FastGRNN-LSQ & {\bf 115.92} & 513 & 12.53 \\ 
 RNN & 144.71 & {\bf 129} & 9.11 \\ 
 SpectralRNN & 130.20 & 242 & - \\ 
 LSTM & 117.41 & 2052 & 13.52 \\ 
 UGRNN & 119.71 & 256 & 11.12 \\ 
 {\bf ERNN} & 119.71 & 529 & {\bf 7.11} \\
 \hline
\end{tabular}
\vspace{-3mm}
\label{table:3}
\end{table}

\begingroup
\setlength{\tabcolsep}{2pt}
\begin{table}[h!]\footnotesize
\caption{Comparison of different RNNs on benchmark datasets}
\vspace{2mm}
\begin{tabular}{|ccccc|} 
 \hline
 Dataset & Algorithm & \begin{tabular}[c]{@{}c@{}}Accuracy\\ (\%)\end{tabular} & \begin{tabular}[c]{@{}c@{}}Model\\ Size (KB)\end{tabular} & \begin{tabular}[c]{@{}c@{}}Train\\ Time (hr)\end{tabular} \\ [0.5ex] 
 \hline\hline
 HAR-2 & FastRNN & 94.50 & 29 & 0.063 \\ 
  & FastGRNN-LSQ & 95.38 & 29 & 0.081 \\ 
  & RNN & 91.31 & 29 & 0.114 \\ 
  & SpectralRNN & 95.48 & 525 & 0.730 \\ 
  & EURNN & 93.11 & \textbf{12} & 0.740 \\ 
  & LSTM & 93.65 & 74 & 0.183 \\ 
  & GRU & 93.62 & 71 & 0.130 \\ 
  & UGRNN & 94.53 & 37 & 0.120 \\ 
  & {\bf ERNN} & \textbf{95.59} & 34 & \textbf{0.061} \\ 
  \hline
 DSA-19 & FastRNN & 84.14 & 97 & 0.032 \\ 
  & FastGRNN-LSQ & 85.00 & 208 & 0.036 \\ 
  & RNN & 71.68 & \textbf{20} & 0.019 \\ 
  & SpectralRNN & 80.37 & 50 & 0.038 \\ 
  & LSTM & 84.84 & 526 & 0.043 \\
  & GRU & 84.84 & 270 & 0.039 \\ 
  & UGRNN & 84.74 & 399 & 0.039 \\ 
  & {\bf ERNN} & \textbf{86.87} & 36 & \textbf{0.015} \\
  \hline
 Google-12 & FastRNN & 92.21 & 56 & 0.61 \\ 
  & FastGRNN-LSQ & 93.18 & 57 & 0.63 \\ 
  & RNN & 73.25 & \textbf{56} & 1.11 \\ 
  & SpectralRNN & 91.59 & 228 & 19.0 \\ 
  & EURNN & 76.79 & 210 & 120.00 \\
  & LSTM & 92.30 & 212 & 1.36 \\ 
  & GRU & 93.15 & 248 & 1.23 \\ 
  & UGRNN & 92.63 & 75 & 0.78 \\ 
  & {\bf ERNN} & \textbf{94.96} & 66 & \textbf{0.20} \\
  \hline
 Google-30 & FastRNN & 91.60 & 96 & 1.30 \\ 
  & FastGRNN-LSQ & 92.03 & \textbf{45} & 1.41 \\ 
  & RNN & 80.05 & 63 & 2.13 \\ 
  & SpectralRNN & 88.73 & 128 & 11.0 \\ 
  & EURNN & 56.35 & 135 & 19.00 \\
  & LSTM & 90.31 & 219 & 2.63 \\ 
  & GRU & 91.41 & 257 & 2.70 \\ 
  & UGRNN & 90.54 & 260 & 2.11 \\ 
  & {\bf ERNN} & \textbf{94.10} & 70 & \textbf{0.44} \\
  \hline
 Pixel-MNIST & FastRNN & 96.44 & 166 & 15.10 \\  
  & FastGRNN-LSQ & \textbf{98.72} & 71 & 12.57 \\ 
  & EURNN & 95.38 & \textbf{64} & 122.00 \\
  & RNN & 94.10 & 71 & 45.56 \\ 
  & LSTM & 97.81 & 265 & 26.57 \\ 
  & GRU & 98.70 & 123 & 23.67 \\ 
  & UGRNN & 97.29 & 84 & 15.17 \\ 
  & {\bf ERNN} & 98.13 & 80 & \textbf{2.17} \\
  \hline
 Yelp-5 & FastRNN & 55.38 & 130 & 3.61 \\  
  & FastGRNN-LSQ & {\bf 59.51} & 130 & 3.91 \\ 
  & RNN & 47.59 & 130 & 3.33 \\ 
  & SpectralRNN & 56.56 & \textbf{89} & 4.92 \\ 
  & EURNN & 59.01 & 122 & 72.00 \\
  & LSTM & 59.49 & 516 & 8.61 \\ 
  & GRU & 59.02 & 388 & 8.12 \\ 
  & UGRNN & 58.67 & 258 & 4.34 \\ 
  & {\bf ERNN} & 57.21 & 138 & \textbf{0.69} \\
 \hline
\end{tabular}
\label{table:2}
\vspace{-4mm}
\end{table}
\endgroup

{\bf Results:}
Table \ref{table:3} and Table \ref{table:2} compare the performance of ERNN to the state-of-the-art RNNs. Four points are worth noticing about ERNN's performance. First, ERNN's prediction gains over a standard RNN ranged from $3.16\%$ on Pixel-MNIST dataset to $21.71\%$ on Google-12 dataset. Similar observations are made for other previously wide-used RNNs as well, demonstrating the superiority of our approach. Second, ERNN's prediction accuracy always surpassed FastRNN's prediction accuracy. This indeed shows the advantage of learning a general $\mathbf{U}$ matrix rather than fixing it to $\mathbf{0}$. Third, ERNN could surpass gating based FastGRNN-LSQ on $4$ out of $6$ dataset in terms of prediction accuracy with $2.87\%$ on DSA-19 dataset and $2.7\%$ on Google-30 dataset. Fourth, and most importantly, ERNN's training speedups over FastRNN as well as FastGRNN-LSQ could range from $1.3$x on HAR-2 dataset to $6$x on Pixel-MNIST dataset. This emphasizes the fact that self-feedback can help ERNN achieve better results than gating backed methods with significantly better training efficiency. Note that the model size of our ERNN is always comparable to, or even better than, the model size of either FastRNN or FastGRNN-LSQ. 

\begin{table}[t]\footnotesize
\centering
\caption{Comparison on different ERNNs on HAR-2 dataset}
\vspace{2mm}
\begin{tabular}{|c c c c|} 
 \hline
 Algorithm & \begin{tabular}[c]{@{}c@{}}Accuracy\\ (\%)\end{tabular} & \begin{tabular}[c]{@{}c@{}}Model\\ Size (KB)\end{tabular} & \begin{tabular}[c]{@{}c@{}}Train\\ Time (hr)\end{tabular} \\ [0.5ex] 
 \hline\hline
 FastRNN & 94.50 & \textbf{29} & 0.063 \\ 
 FastGRNN-LSQ & 95.38 & \textbf{29} & 0.081 \\ 
 RNN & 91.31 & \textbf{29} & 0.114 \\ 
 LSTM & 93.65 & 74 & 0.183 \\ 
 {\bf ERNN(K=1)} & 95.59 & 34 & \textbf{0.061} \\ 
 {\bf ERNN(K=2)} & \textbf{96.33} & 35 & 0.083 \\ 
 \hline
\end{tabular}
\vspace{-3mm}
\label{table:4}
\end{table}

\begin{figure}[t]
	\centerline{\includegraphics[clip=true,width=.9\linewidth]{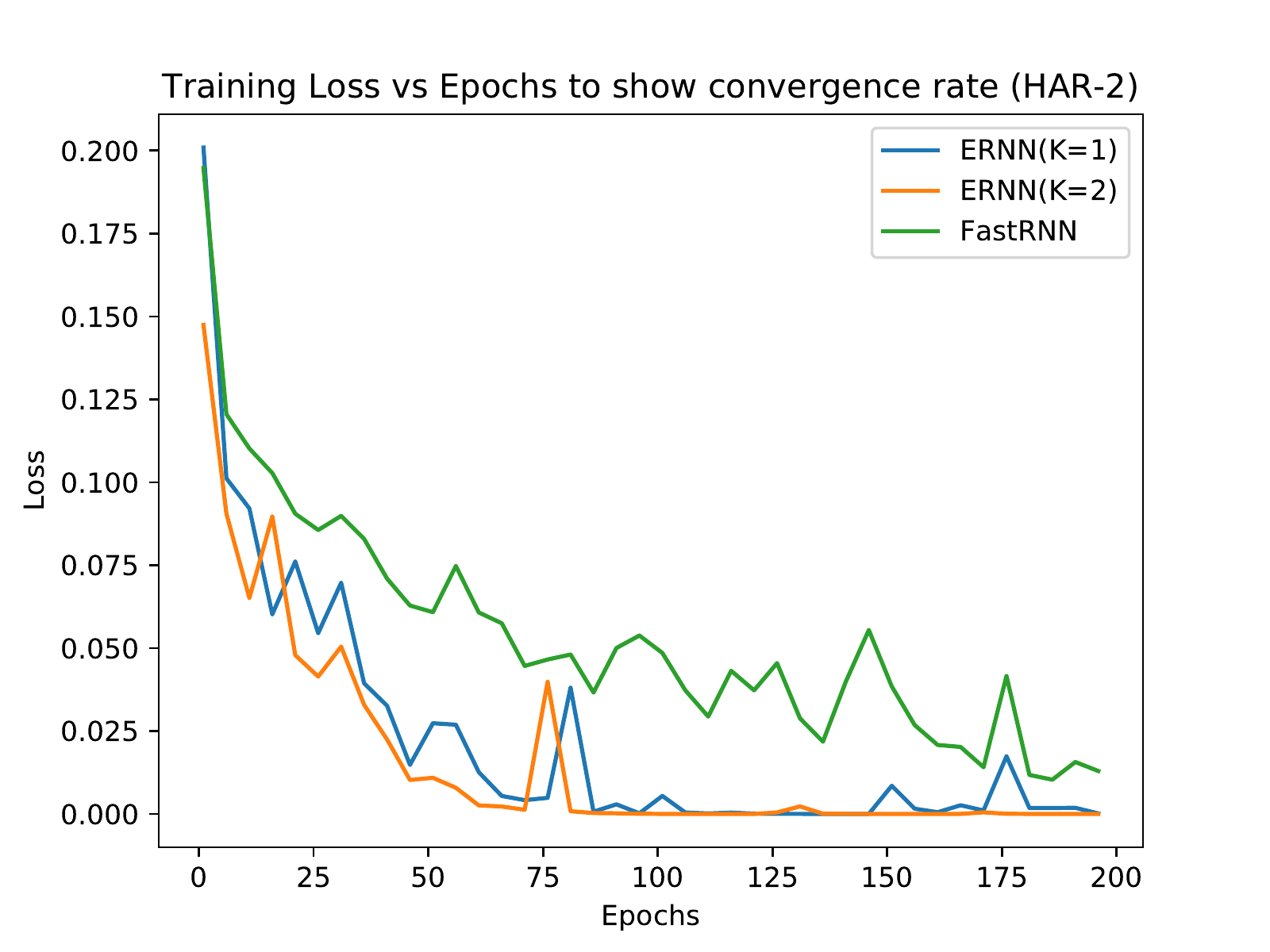}}
    \vspace{-3mm}
	\caption{\footnotesize Comparison on convergence of different approaches.}
	\label{fig:acc_vs_epoch_har_2}
	\vspace{-3mm}
\end{figure}

\begin{figure}[t]
	\begin{minipage}[b]{0.495\linewidth}
		\begin{center}
			\centerline{\includegraphics[clip=true,width=1.12\linewidth]{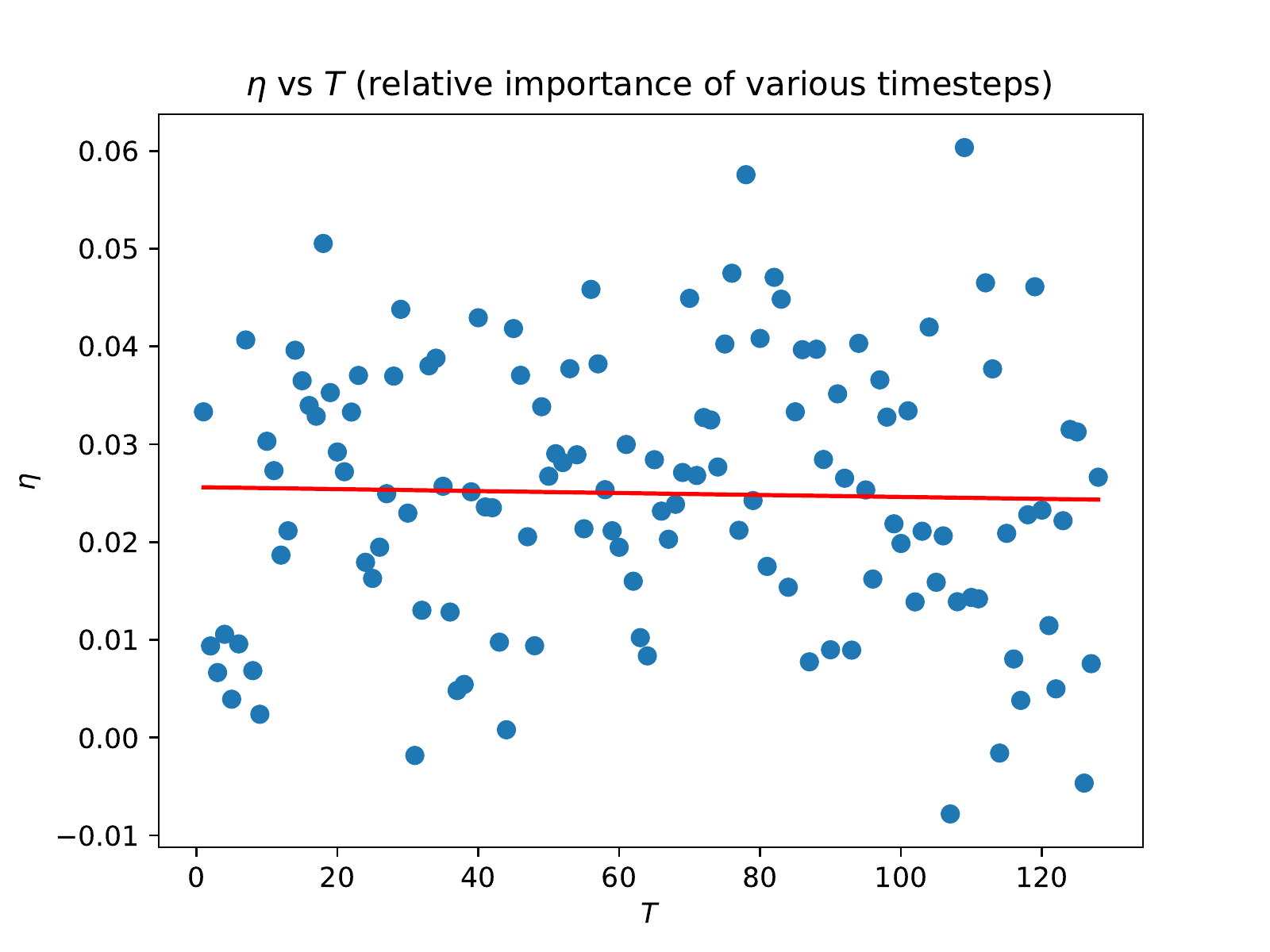}}	
			\centerline{\footnotesize (a) $K=1$ on HAR-2}	
		\end{center}
	\end{minipage}
	\begin{minipage}[b]{0.495\linewidth}
		\begin{center}
			\centerline{\includegraphics[clip=true,width=1.12\linewidth]{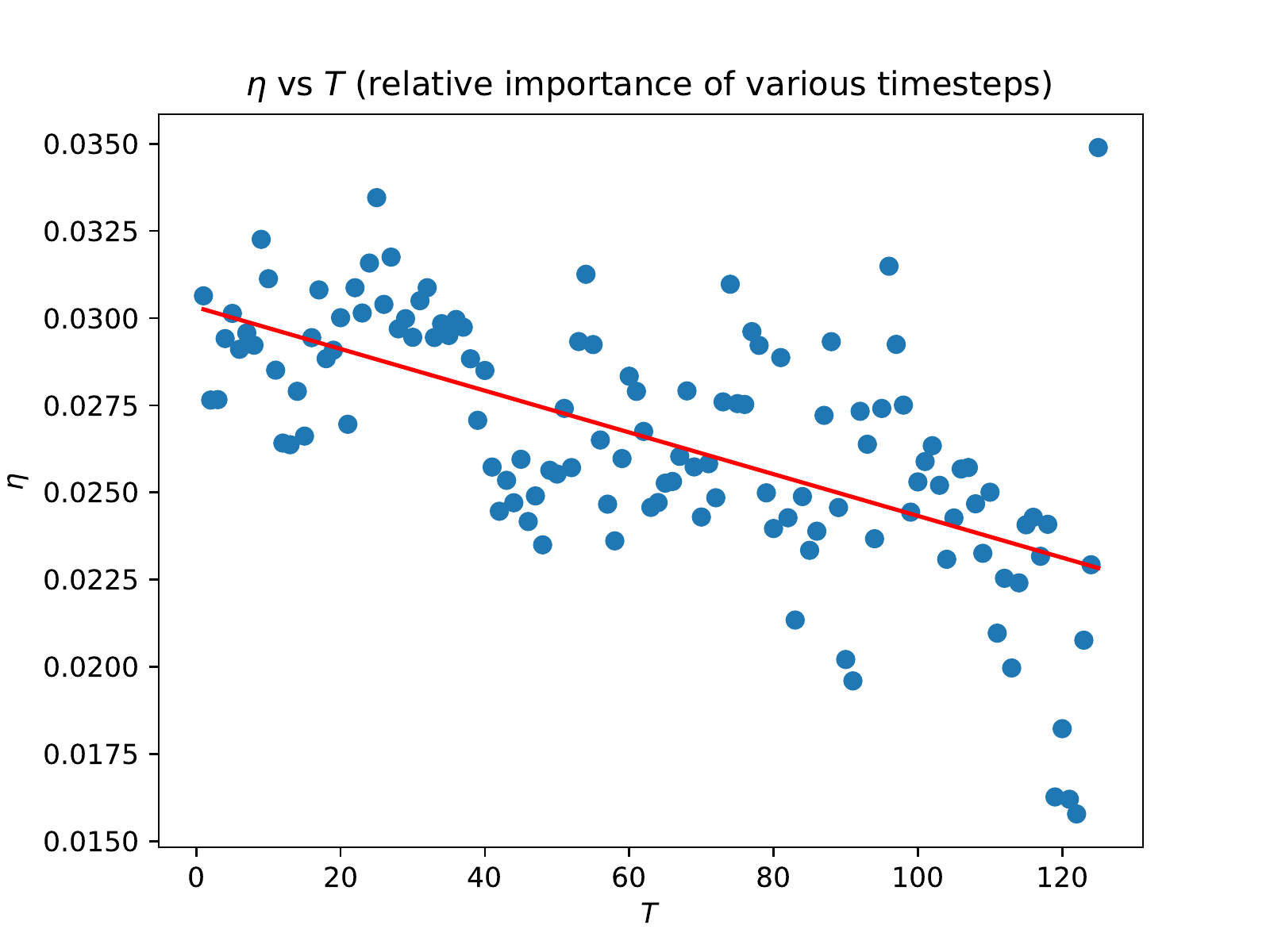}}	
			\centerline{\footnotesize (b) $K=1$ on DSA-19}	
		\end{center}
	\end{minipage}
	\begin{minipage}[b]{0.495\linewidth}
		\begin{center}
			\centerline{\includegraphics[clip=true,width=1.12\linewidth]{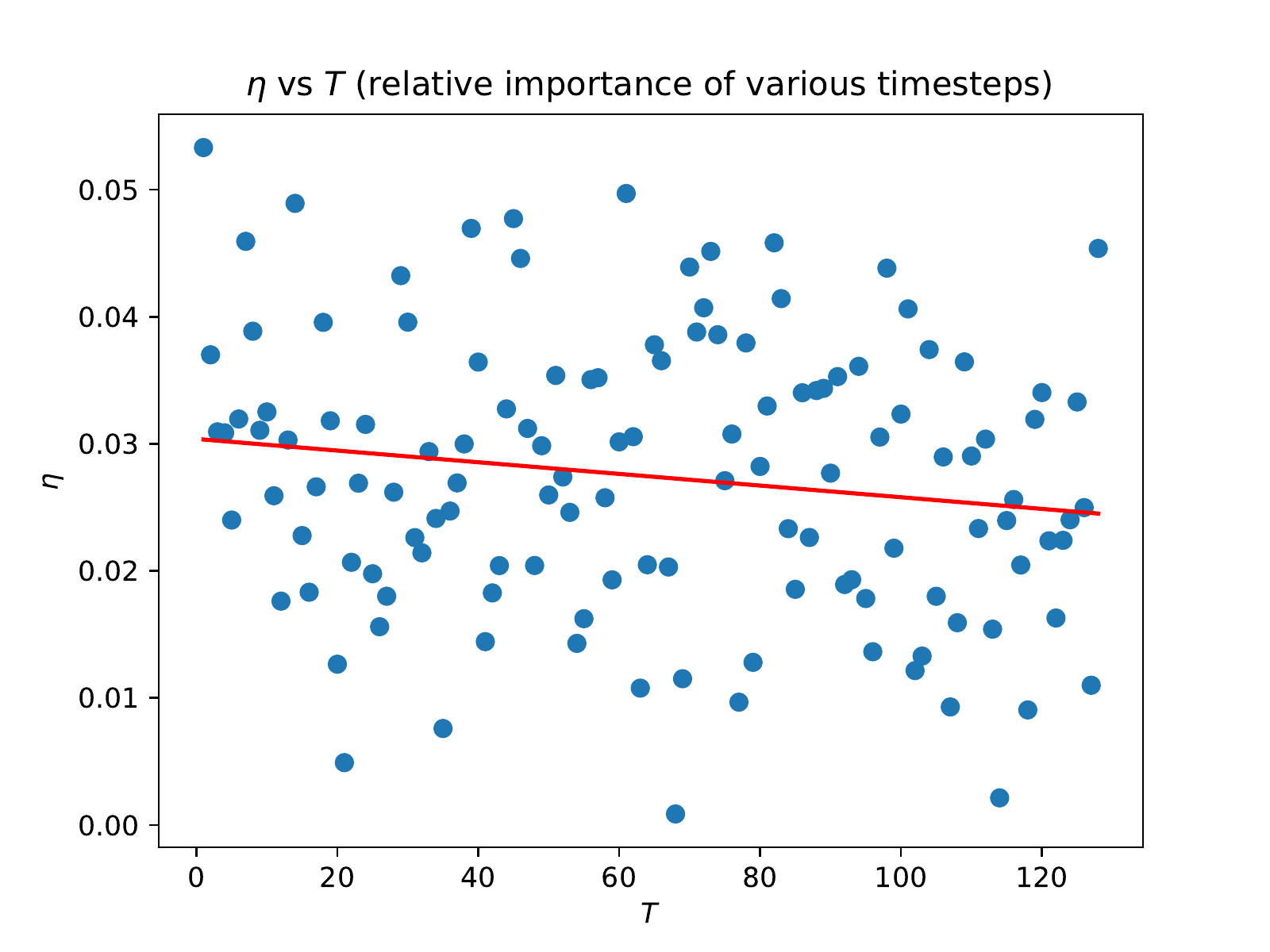}}	
			\centerline{\footnotesize (c) $k=1, K=2$ on HAR-2}	
		\end{center}
	\end{minipage}
	\begin{minipage}[b]{0.495\linewidth}
		\begin{center}
			\centerline{\includegraphics[clip=true,width=1.12\linewidth]{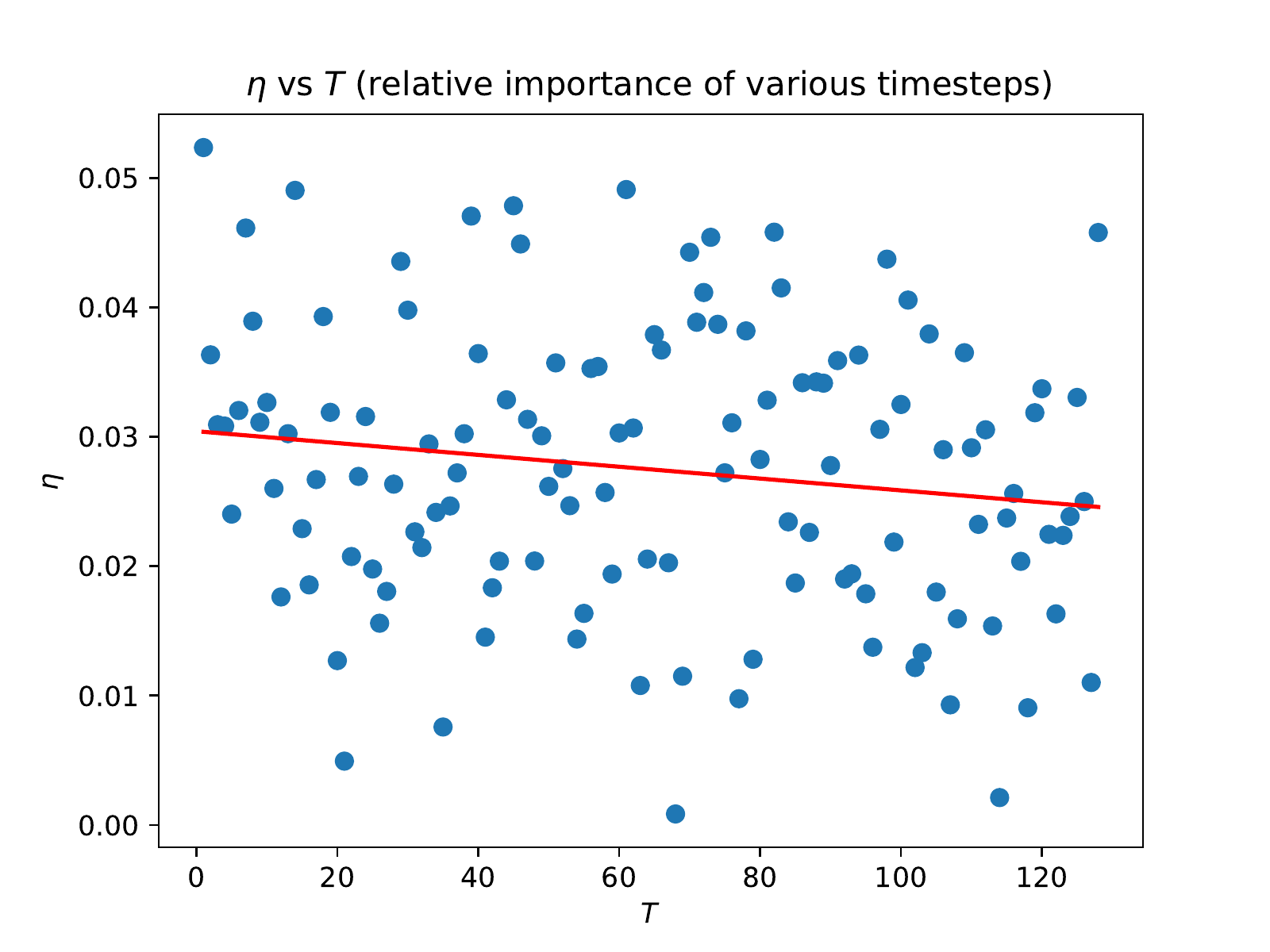}}	
			\centerline{\footnotesize (d) $k=2, K=2$ on HAR-2}	
		\end{center}
	\end{minipage}
	\vspace{-7mm}
	\caption{\footnotesize Illustration of learned $\eta$'s over timesteps.}
	\label{fig:eta}
\end{figure}


Table \ref{table:4} compares the performance of ERNN with K=1 to that of K=2 (see Fig. \ref{fig:resRNN} for the definition of $K$) on HAR-2 dataset. It can be seen that ERNN(K=2) achieves nearly $1\%$ higher prediction accuracy with almost no change in model size but slightly higher training time. In comparison to FastGRNN-LSQ, ERRN(K=2) achieves $1\%$ higher prediction accuracy with similar training time.

To further verify the advantage of our approach on convergence, we show the training behavior of different approaches in Fig. \ref{fig:acc_vs_epoch_har_2}. As we see, our ERNNs converge significantly faster than FastRNN while achieving lower losses. This observation demonstrates the importance of locating equilibrium points for dynamical systems. Meanwhile, the curve for ERNN(K=2) tends to stay below the curve for ERNN(K=1), indicating that finer approximate solutions for equilibrium points lead to faster convergence as well as better generalization (see Table \ref{table:3}). The training time reported in Table \ref{table:3}, Table \ref{table:2}, and Table \ref{table:4} is the one for achieving a convergent model with best accuracy.

To understand the effect of learnable weights $\eta$'s, we show these learned values in Fig. \ref{fig:eta}, where we fit red lines to the scatter plots of these ($\eta$, $T$) pairs based on least squares. Overall all the numbers here are very small, which is very useful to make the condition in Thm. \ref{thm:convergence} feasible in practice. In general, we observe similar decreasing behavior to that reported in \cite{NIPS2018_8116}. In contrast, our learning procedure does not constrain $\eta$ to be positive, and thus we can learn some negatives to better fit the supervision signals, as illustrated in Fig. \ref{fig:eta}(a). Across different datasets, $\eta$'s form different patterns. For the case of $K=2$ on the same dataset, however, we observe that the patterns of learned $\eta$'s are almost identical with slightly change in values. This indicates that we may just need to learn a single $\eta$ for different $k$'s. We will investigate more on the impact of $\eta$'s on accuracy in our future work.

\section{Conclusion}
Motivated by autapse in neuroscience, we propose a novel {\em Equilibrated Recurrent Neural Network (ERNN)} method. We introduce neuronal time-delayed self-feedback into conventional recurrent models, which leads to better generalization as well as efficient training. We demonstrate empirically that such neuronal self-feedback helps stabilize the hidden state transition matrices rapidly by learning discriminative latent features, leading to fast convergence in training and good accuracy in testing. To locate fixed points efficiently, we propose a novel inexact Newton method that can be proven to converge locally with linear rate (under mild conditions). As a result we can recast ERNN training into networks with residual connections in a principled way and train them efficiently based on backpropagation. We demonstrate the superiority of ERNNs over the state-of-the-art on several benchmark datasets in terms of both accuracy and training efficiency. For instance, on the Google-30 dataset our ERNN outperforms FastRNN by 2.50\% in accuracy with $\sim3\times$ faster training speed.

\newpage
\small
\bibliography{example_paper}

\begin{thebibliography}{49}
\providecommand{\natexlab}[1]{#1}
\providecommand{\url}[1]{\texttt{#1}}
\expandafter\ifx\csname urlstyle\endcsname\relax
  \providecommand{\doi}[1]{doi: #1}\else
  \providecommand{\doi}{doi: \begingroup \urlstyle{rm}\Url}\fi

\bibitem[Altun et~al.(2010)Altun, Barshan, and Tun\c{c}el]{Altun2010DSA19}
Altun, K., Barshan, B., and Tun\c{c}el, O.
\newblock Comparative study on classifying human activities with miniature
  inertial and magnetic sensors.
\newblock \emph{Pattern Recogn.}, 43\penalty0 (10):\penalty0 3605--3620,
  October 2010.
\newblock ISSN 0031-3203.
\newblock \doi{10.1016/j.patcog.2010.04.019}.
\newblock URL \url{http://dx.doi.org/10.1016/j.patcog.2010.04.019}.

\bibitem[Anguita et~al.(2012)Anguita, Ghio, Oneto, Parra, and
  Reyes-Ortiz]{Anguita2012HAR}
Anguita, D., Ghio, A., Oneto, L., Parra, X., and Reyes-Ortiz, J.~L.
\newblock Human activity recognition on smartphones using a multiclass
  hardware-friendly support vector machine.
\newblock In \emph{Proceedings of the 4th International Conference on Ambient
  Assisted Living and Home Care}, IWAAL'12, pp.\  216--223, Berlin, Heidelberg,
  2012. Springer-Verlag.
\newblock ISBN 978-3-642-35394-9.
\newblock \doi{10.1007/978-3-642-35395-6_30}.
\newblock URL \url{http://dx.doi.org/10.1007/978-3-642-35395-6_30}.

\bibitem[Arjovsky et~al.(2016)Arjovsky, Shah, and Bengio]{arjovsky2016unitary}
Arjovsky, M., Shah, A., and Bengio, Y.
\newblock Unitary evolution recurrent neural networks.
\newblock In \emph{International Conference on Machine Learning}, pp.\
  1120--1128, 2016.

\bibitem[Barabanov \& Prokhorov(2002)Barabanov and Prokhorov]{Barabanov}
Barabanov, N.~E. and Prokhorov, D.~V.
\newblock Stability analysis of discrete-time recurrent neural networks.
\newblock \emph{Trans. Neur. Netw.}, 13\penalty0 (2):\penalty0 292--303, March
  2002.
\newblock ISSN 1045-9227.
\newblock \doi{10.1109/72.991416}.
\newblock URL \url{https://doi.org/10.1109/72.991416}.

\bibitem[Barzilai \& Borwein(1988)Barzilai and Borwein]{barzilai1988two}
Barzilai, J. and Borwein, J.~M.
\newblock Two-point step size gradient methods.
\newblock \emph{IMA journal of numerical analysis}, 8\penalty0 (1):\penalty0
  141--148, 1988.

\bibitem[Bradbury et~al.(2016)Bradbury, Merity, Xiong, and
  Socher]{bradbury2016quasi}
Bradbury, J., Merity, S., Xiong, C., and Socher, R.
\newblock Quasi-recurrent neural networks.
\newblock \emph{arXiv preprint arXiv:1611.01576}, 2016.

\bibitem[Campos et~al.(2017)Campos, Jou, Gir{\'o}-i Nieto, Torres, and
  Chang]{campos2017skip}
Campos, V., Jou, B., Gir{\'o}-i Nieto, X., Torres, J., and Chang, S.-F.
\newblock Skip rnn: Learning to skip state updates in recurrent neural
  networks.
\newblock \emph{arXiv preprint arXiv:1708.06834}, 2017.

\bibitem[Chang et~al.(2017)Chang, Zhang, Han, Yu, Guo, Tan, Cui, Witbrock,
  Hasegawa-Johnson, and Huang]{chang2017dilated}
Chang, S., Zhang, Y., Han, W., Yu, M., Guo, X., Tan, W., Cui, X., Witbrock, M.,
  Hasegawa-Johnson, M.~A., and Huang, T.~S.
\newblock Dilated recurrent neural networks.
\newblock In \emph{Advances in Neural Information Processing Systems}, pp.\
  77--87, 2017.

\bibitem[Cho et~al.(2014)Cho, Van~Merri{\"e}nboer, Bahdanau, and
  Bengio]{cho2014properties}
Cho, K., Van~Merri{\"e}nboer, B., Bahdanau, D., and Bengio, Y.
\newblock On the properties of neural machine translation: Encoder-decoder
  approaches.
\newblock \emph{arXiv preprint arXiv:1409.1259}, 2014.

\bibitem[Cho et~al.(2015)Cho, Dhir, and Lee]{cho2015hessian}
Cho, M., Dhir, C., and Lee, J.
\newblock Hessian-free optimization for learning deep multidimensional
  recurrent neural networks.
\newblock In \emph{Advances in Neural Information Processing Systems}, pp.\
  883--891, 2015.

\bibitem[{Collins} et~al.(2016){Collins}, {Sohl-Dickstein}, and
  {Sussillo}]{2016arXivUGRNN}
{Collins}, J., {Sohl-Dickstein}, J., and {Sussillo}, D.
\newblock {Capacity and Trainability in Recurrent Neural Networks}.
\newblock \emph{arXiv e-prints}, art. arXiv:1611.09913, November 2016.

\bibitem[Dauphin et~al.(2017)Dauphin, Fan, Auli, and
  Grangier]{Dauphin2017LanguageMW}
Dauphin, Y., Fan, A., Auli, M., and Grangier, D.
\newblock Language modeling with gated convolutional networks.
\newblock In \emph{ICML}, 2017.

\bibitem[Dembo et~al.(1982)Dembo, Eisenstat, and Steihaug]{dembo1982inexact}
Dembo, R.~S., Eisenstat, S.~C., and Steihaug, T.
\newblock Inexact newton methods.
\newblock \emph{SIAM Journal on Numerical analysis}, 19\penalty0 (2):\penalty0
  400--408, 1982.

\bibitem[Fan et~al.(2018)Fan, Wang, Wang, Lai, and Wang]{fan2018autapses}
Fan, H., Wang, Y., Wang, H., Lai, Y.-C., and Wang, X.
\newblock Autapses promote synchronization in neuronal networks.
\newblock \emph{Scientific reports}, 8\penalty0 (1):\penalty0 580, 2018.

\bibitem[Gehring et~al.(2017)Gehring, Auli, Grangier, Yarats, and
  Dauphin]{Gehring2017ConvolutionalST}
Gehring, J., Auli, M., Grangier, D., Yarats, D., and Dauphin, Y.
\newblock Convolutional sequence to sequence learning.
\newblock In \emph{ICML}, 2017.

\bibitem[Gruslys et~al.(2016)Gruslys, Munos, Danihelka, Lanctot, and
  Graves]{NIPS2016_6221}
Gruslys, A., Munos, R., Danihelka, I., Lanctot, M., and Graves, A.
\newblock Memory-efficient backpropagation through time.
\newblock In Lee, D.~D., Sugiyama, M., Luxburg, U.~V., Guyon, I., and Garnett,
  R. (eds.), \emph{Advances in Neural Information Processing Systems 29}, pp.\
  4125--4133. 2016.

\bibitem[Herrmann \& Klaus(2004)Herrmann and Klaus]{herrmann2004autapse}
Herrmann, C.~S. and Klaus, A.
\newblock Autapse turns neuron into oscillator.
\newblock \emph{International Journal of Bifurcation and Chaos}, 14\penalty0
  (02):\penalty0 623--633, 2004.

\bibitem[Hochreiter \& Schmidhuber(1997)Hochreiter and
  Schmidhuber]{hochreiter1997long}
Hochreiter, S. and Schmidhuber, J.
\newblock Long short-term memory.
\newblock \emph{Neural computation}, 9\penalty0 (8):\penalty0 1735--1780, 1997.

\bibitem[Hori et~al.(2017)Hori, Hori, Lee, Zhang, Harsham, Hershey, Marks, and
  Sumi]{hori2017attention}
Hori, C., Hori, T., Lee, T.-Y., Zhang, Z., Harsham, B., Hershey, J.~R., Marks,
  T.~K., and Sumi, K.
\newblock Attention-based multimodal fusion for video description.
\newblock In \emph{ICCV}, pp.\  4203--4212, 2017.

\bibitem[Jaeger(2002)]{jaeger2002tutorial}
Jaeger, H.
\newblock \emph{Tutorial on training recurrent neural networks, covering BPPT,
  RTRL, EKF and the" echo state network" approach}, volume~5.
\newblock 2002.

\bibitem[Jing et~al.(2017)Jing, Shen, Dubcek, Peurifoy, Skirlo, LeCun, Tegmark,
  and Solja{\v{c}}i{\'c}]{jing2017tunable}
Jing, L., Shen, Y., Dubcek, T., Peurifoy, J., Skirlo, S., LeCun, Y., Tegmark,
  M., and Solja{\v{c}}i{\'c}, M.
\newblock Tunable efficient unitary neural networks (eunn) and their
  application to rnns.
\newblock In \emph{International Conference on Machine Learning}, pp.\
  1733--1741, 2017.

\bibitem[{Jing} et~al.(2017){Jing}, {Shen}, {Dub{\v{c}}ek}, {Peurifoy},
  {Skirlo}, {LeCun}, {Tegmark}, and {Solja{\v{c}}i{\'c}}]{2017EURNN}
{Jing}, L., {Shen}, Y., {Dub{\v{c}}ek}, T., {Peurifoy}, J., {Skirlo}, S.,
  {LeCun}, Y., {Tegmark}, M., and {Solja{\v{c}}i{\'c}}, M.
\newblock Tunable efficient unitary neural networks (eunn) and their
  application to rnns.
\newblock In \emph{ICML}, 2017.

\bibitem[Jordan(1997)]{jordan1997serial}
Jordan, M.~I.
\newblock Serial order: A parallel distributed processing approach.
\newblock In \emph{Advances in psychology}, volume 121, pp.\  471--495.
  Elsevier, 1997.

\bibitem[{Kingma} \& {Ba}(2015){Kingma} and {Ba}]{2015Adam}
{Kingma}, D.~P. and {Ba}, J.
\newblock Adam: A method for stochastic optimization.
\newblock In \emph{ICML}, 2015.

\bibitem[Kusupati et~al.(2018{\natexlab{a}})Kusupati, Singh, Bhatia, Kumar,
  Jain, and Varma]{NIPS2018_8116}
Kusupati, A., Singh, M., Bhatia, K., Kumar, A., Jain, P., and Varma, M.
\newblock Fastgrnn: A fast, accurate, stable and tiny kilobyte sized gated
  recurrent neural network.
\newblock In \emph{Advances in Neural Information Processing Systems 31}, pp.\
  9031--9042. 2018{\natexlab{a}}.

\bibitem[Kusupati et~al.(2018{\natexlab{b}})Kusupati, Singh, Bhatia, Kumar,
  Jain, and Varma]{kusupati2018nips}
Kusupati, A., Singh, M., Bhatia, K., Kumar, A., Jain, P., and Varma, M.
\newblock Fastgrnn: A fast, accurate, stable and tiny kilobyte sized gated
  recurrent neural network.
\newblock In \emph{Advances in Neural Information Processing Systems},
  2018{\natexlab{b}}.

\bibitem[Lecun et~al.(1998)Lecun, Bottou, Bengio, and
  Haffner]{Lecun98gradient-basedlearning}
Lecun, Y., Bottou, L., Bengio, Y., and Haffner, P.
\newblock Gradient-based learning applied to document recognition.
\newblock In \emph{Proceedings of the IEEE}, pp.\  2278--2324, 1998.

\bibitem[Lei et~al.(2017)Lei, Zhang, and Artzi]{lei2017training}
Lei, T., Zhang, Y., and Artzi, Y.
\newblock Training rnns as fast as cnns.
\newblock \emph{arXiv preprint arXiv:1709.02755}, 2017.

\bibitem[Liao et~al.(2018)Liao, Xiong, Fetaya, Zhang, Yoon, Pitkow, Urtasun,
  and Zemel]{liao2018reviving}
Liao, R., Xiong, Y., Fetaya, E., Zhang, L., Yoon, K., Pitkow, X., Urtasun, R.,
  and Zemel, R.
\newblock Reviving and improving recurrent back-propagation.
\newblock \emph{arXiv preprint arXiv:1803.06396}, 2018.

\bibitem[MacKay et~al.(2018)MacKay, Vicol, Ba, and
  Grosse]{mackay2018reversible}
MacKay, M., Vicol, P., Ba, J., and Grosse, R.~B.
\newblock Reversible recurrent neural networks.
\newblock In \emph{Advances in Neural Information Processing Systems}, pp.\
  9043--9054, 2018.

\bibitem[Martens \& Sutskever(2011)Martens and Sutskever]{martens2011learning}
Martens, J. and Sutskever, I.
\newblock Learning recurrent neural networks with hessian-free optimization.
\newblock In \emph{Proceedings of the 28th International Conference on Machine
  Learning (ICML-11)}, pp.\  1033--1040, 2011.

\bibitem[Martin \& Cundy(2018)Martin and Cundy]{martin2018parallelizing}
Martin, E. and Cundy, C.
\newblock Parallelizing linear recurrent neural nets over sequence length.
\newblock In \emph{International Conference on Learning Representations}, 2018.

\bibitem[McAuley \& Leskovec(2013)McAuley and Leskovec]{McAuley2013PTB}
McAuley, J. and Leskovec, J.
\newblock Hidden factors and hidden topics: Understanding rating dimensions
  with review text.
\newblock In \emph{Proceedings of the 7th ACM Conference on Recommender
  Systems}, RecSys '13, pp.\  165--172, New York, NY, USA, 2013. ACM.
\newblock ISBN 978-1-4503-2409-0.
\newblock \doi{10.1145/2507157.2507163}.
\newblock URL \url{http://doi.acm.org/10.1145/2507157.2507163}.

\bibitem[Microsoft(2018)]{EdgeML2018}
Microsoft.
\newblock Edge machine learning.
\newblock 2018.
\newblock URL \url{https://github.com/Microsoft/EdgeML}.

\bibitem[Miller \& Hardt(2018)Miller and Hardt]{miller2018recurrent}
Miller, J. and Hardt, M.
\newblock When recurrent models don't need to be recurrent.
\newblock \emph{arXiv preprint arXiv:1805.10369}, 2018.

\bibitem[Mujika et~al.(2018)Mujika, Meier, and Steger]{NIPS2018_7894}
Mujika, A., Meier, F., and Steger, A.
\newblock Approximating real-time recurrent learning with random kronecker
  factors.
\newblock In \emph{Advances in Neural Information Processing Systems 31}, pp.\
  6594--6603. 2018.

\bibitem[Pascanu et~al.(2013)Pascanu, Mikolov, and
  Bengio]{pascanu2013difficulty}
Pascanu, R., Mikolov, T., and Bengio, Y.
\newblock On the difficulty of training recurrent neural networks.
\newblock In \emph{International Conference on Machine Learning}, pp.\
  1310--1318, 2013.

\bibitem[Qi et~al.(2017)Qi, Su, Mo, and Guibas]{qi2017pointnet}
Qi, C.~R., Su, H., Mo, K., and Guibas, L.~J.
\newblock Pointnet: Deep learning on point sets for 3d classification and
  segmentation.
\newblock In \emph{CVPR}, pp.\  652--660, 2017.

\bibitem[Qin et~al.(2015)Qin, Wu, Wang, and Ma]{qin2015emitting}
Qin, H., Wu, Y., Wang, C., and Ma, J.
\newblock Emitting waves from defects in network with autapses.
\newblock \emph{Communications in Nonlinear Science and Numerical Simulation},
  23\penalty0 (1-3):\penalty0 164--174, 2015.

\bibitem[Seung et~al.(2000)Seung, Lee, Reis, and Tank]{seung2000autapse}
Seung, H.~S., Lee, D.~D., Reis, B.~Y., and Tank, D.~W.
\newblock The autapse: a simple illustration of short-term analog memory
  storage by tuned synaptic feedback.
\newblock \emph{Journal of computational neuroscience}, 9\penalty0
  (2):\penalty0 171--185, 2000.

\bibitem[Sutskever et~al.(2014)Sutskever, Vinyals, and
  Le]{sutskever2014sequence}
Sutskever, I., Vinyals, O., and Le, Q.~V.
\newblock Sequence to sequence learning with neural networks.
\newblock In \emph{Advances in neural information processing systems}, pp.\
  3104--3112, 2014.

\bibitem[Tallec \& Ollivier(2017)Tallec and Ollivier]{tallec2017unbiased}
Tallec, C. and Ollivier, Y.
\newblock Unbiased online recurrent optimization.
\newblock \emph{arXiv preprint arXiv:1702.05043}, 2017.

\bibitem[van~den Oord et~al.(2016)van~den Oord, Dieleman, Zen, Simonyan,
  Vinyals, Graves, Kalchbrenner, Senior, and Kavukcuoglu]{Oord2016WaveNetAG}
van~den Oord, A., Dieleman, S., Zen, H., Simonyan, K., Vinyals, O., Graves, A.,
  Kalchbrenner, N., Senior, A.~W., and Kavukcuoglu, K.
\newblock Wavenet: A generative model for raw audio.
\newblock In \emph{SSW}, 2016.

\bibitem[Vaswani et~al.(2017)Vaswani, Shazeer, Parmar, Uszkoreit, Jones, Gomez,
  Kaiser, and Polosukhin]{Vaswani2017AttentionIA}
Vaswani, A., Shazeer, N., Parmar, N., Uszkoreit, J., Jones, L., Gomez, A.~N.,
  Kaiser, L., and Polosukhin, I.
\newblock Attention is all you need.
\newblock In \emph{NIPS}, 2017.

\bibitem[{Warden}(2018)]{warden2018google30}
{Warden}, P.
\newblock {Speech Commands: A Dataset for Limited-Vocabulary Speech
  Recognition}.
\newblock \emph{arXiv e-prints}, art. arXiv:1804.03209, April 2018.

\bibitem[Werbos(1990)]{werbos1990backpropagation}
Werbos, P.~J.
\newblock Backpropagation through time: what it does and how to do it.
\newblock \emph{Proceedings of the IEEE}, 78\penalty0 (10):\penalty0
  1550--1560, 1990.

\bibitem[Williams \& Zipser(1989)Williams and Zipser]{williams1989learning}
Williams, R.~J. and Zipser, D.
\newblock A learning algorithm for continually running fully recurrent neural
  networks.
\newblock \emph{Neural computation}, 1\penalty0 (2):\penalty0 270--280, 1989.

\bibitem[Yelp(2017)]{Yelp2017}
Yelp, I.
\newblock Yelp dataset challenge.
\newblock 2017.
\newblock URL \url{https://www.yelp.com/dataset/challenge}.

\bibitem[{Zhang} et~al.(2018){Zhang}, {Lei}, and {Dhillon}]{2018SpectralRNN}
{Zhang}, J., {Lei}, Q., and {Dhillon}, I.~S.
\newblock Stabilizing gradients for deep neural networks via efficient svd
  parameterization.
\newblock In \emph{ICML}, 2018.

\end{thebibliography}
\bibliographystyle{icml2019}

\end{document}